\documentclass[10pt,journal,compsoc]{IEEEtran}
\usepackage[utf8]{inputenc} 
\usepackage[T1]{fontenc}    
\usepackage{multirow}
\usepackage{algorithm,algorithmic}

\usepackage{makecell}
\usepackage{threeparttable}

\usepackage{epsfig}

\usepackage{booktabs}       

\usepackage[dvipsnames, svgnames, x11names, table]{xcolor}
\usepackage{arydshln}
\usepackage{colortbl}
\usepackage{amsmath}
\usepackage{amssymb}
\usepackage{amsthm}
\usepackage{pifont}
\usepackage{subfig}
\usepackage{multirow}

\usepackage{hyperref}

\definecolor{red2}{RGB}{240,191,211}
\definecolor{deepred}{RGB}{219,95,146}
\definecolor{blue2}{RGB}{205,226,247}
\definecolor{deepblue}{RGB}{88,161,230}

\theoremstyle{plain}
\newtheorem{theorem}{Theorem} 
  
\newtheorem{lemma}[theorem]{Lemma}

\newtheorem{definition}{Definition}

\definecolor{mygray}{gray}{.9}
\definecolor{myred}{HTML}{C81D31}
\definecolor{myblue}{HTML}{4874CB}

\makeatletter
\def\hlinew#1{%
\noalign{\ifnum0=`}\fi\hrule \@height #1 \futurelet
\reserved@a\@xhline}
\makeatother%

\begin{document}
\title{Adaptive Point-Prompt Tuning: Fine-Tuning Heterogeneous Foundation Models for \\3D Point Cloud Analysis} 

\author{Mengke Li,
        Lihao Chen,
        Peng Zhang,
        Yiu-ming~Cheung,~\IEEEmembership{Fellow,~IEEE,}
        Hui Huang* ~\IEEEmembership{Senior Member,~IEEE}
\IEEEcompsocitemizethanks{
\IEEEcompsocthanksitem Mengke Li is with the College of Computer Science and Software Engineering, Shenzhen University, Shenzhen, China (e-mail: mengkejiajia@hotmail.com). 
\IEEEcompsocthanksitem Lihao Chen is with Guangdong Laboratory of Artificial Intelligence and Digital Economy (SZ), Shenzhen, China (e-mail: clihao254@gmail.com). 
\IEEEcompsocthanksitem Peng Zhang is with National Laboratory of Radar Signal Processing, Xidian University, Xi'an, China (e-mail: pzhang@xidian.edu.cn). 
\IEEEcompsocthanksitem Yiu-ming Cheung is with the Department of Computer Science, Hong Kong Baptist University, Hong Kong SAR, China (e-mail: ymc@comp.hkbu.edu.hk). 
\IEEEcompsocthanksitem Hui Huang is the corresponding author with the College of Computer Science and Software Engineering, Shenzhen University, Shenzhen, China (e-mail: hhzhiyan@gmail.com).
}}

\markboth{Submitted to IEEE Transactions on Pattern Analysis and Machine Intelligence}{M. Li \MakeLowercase{\textit{et al.}}: 
Adaptive Point-Prompt Tuning: Adapting Heterogeneous Foundation Models for 3D Point Cloud Analysis}

\maketitle
\begin{abstract}
Parameter-efficient fine-tuning strategies for foundation models in 1D textual and 2D visual analysis have demonstrated remarkable efficacy.
However, due to the scarcity of point cloud data, pre-training large 3D models remains a challenging task. 
While many efforts have been made to apply pre-trained visual models to 3D domains through "high-to-low" mapping, these approaches often lead to the loss of spatial geometries and lack a generalizable framework for adapting any modality to 3D.
This paper, therefore, attempts to directly leverage point features to calibrate the heterogeneous foundation model of any modality for 3D point cloud analysis. 
Specifically, we propose the Adaptive Point-Prompt Tuning (APPT) method, which fine-tunes pre-trained models with a modest number of parameters, enabling direct point cloud processing without heterogeneous mappings.
We convert raw point clouds into point embeddings by aggregating local geometry to capture spatial features followed by linear layers to ensure seamless utilization of frozen pre-trained models.
Given the inherent disorder of point clouds, in contrast to the structured nature of images and language, we employ a permutation-invariant feature to capture the relative positions of point embeddings, thereby obtaining point tokens enriched with location information to optimize self-attention mechanisms.
To calibrate self-attention across source domains of any modality to 3D and reduce computational overhead, we introduce a prompt generator that shares weights with the point embedding module, dynamically producing point-prompts without adding additional parameters.
These prompts are then concatenated into a frozen foundation model, providing rich global structural information and compensating for the lack of structural context in the heterogeneous data.
Extensive experiments on multiple benchmarks demonstrate that our APPT is effective for various downstream tasks in point cloud analysis while achieving high efficiency by fine-tuning only 3.8\% of the trainable parameters. 
The source code and additional details are available at \url{https://github.com/wish254/APPT}.
\end{abstract}

\begin{IEEEkeywords}
Point cloud analysis, 3D vision, parameter-efficient fine-tuning, fine-tuning foundation models.
\end{IEEEkeywords}

\section{Introduction}
\label{sec:intro}

Parameter-efficient fine-tuning (PEFT)~\cite{chen2020tuning, HuSWALWWC22LoRA, chen2022adaptformer, yu2023visual} has emerged as a widely adopted strategy for leveraging the rich semantic features and representation capabilities of large foundation models across diverse downstream tasks, while simultaneously reducing computational and storage costs~\cite{liu2023pre}.
This progress has been particularly notable in the fields of natural language processing (NLP) \cite{devlin2018bert, Floridi2020GPT-3} and computer vision (CV)~\cite{Dosovitskiy21vit, radford2021clip, Oquab2023DINOv2LR}, where the growing availability of training data has led to the continuous emergence of pre-trained foundation models.
However, 3D visual understanding~\cite{guo2020deep}, as an important research topic, faces significantly greater challenges in data acquisition compared to NLP and CV. 
This results in a lack of large-scale foundation models for 3D tasks. 
Although several 3D pre-trained models, such as Point-BERT~\cite{yu2022point}, OcCo~\cite{wangHC2021occo}, and PointGPT~\cite{NEURIPS2023pointgpt}, have shown promising results, their scale remains incomparable to models trained on image or text data. 
For instance, the 3D foundation model, PointGPT-L~\cite{NEURIPS2023pointgpt} is pre-trained on a multi-source dataset containing approximately 3 million point clouds, whereas the visual-linguistic model CLIP~\cite{radford2021clip} is trained on 400 million image-text pairs.
Acquiring and annotating real high-quality 3D data requires significant resources and human labor, and synthetic 3D data often lacks distribution diversity and real-world applicability~\cite{tang2025any2point}. 
These limitations raise the question of whether prior knowledge from 2D or 1D data can be effectively leveraged for the analysis of 3D point clouds.

\begin{figure}[!t]
    \centering
    \includegraphics[width=\linewidth]{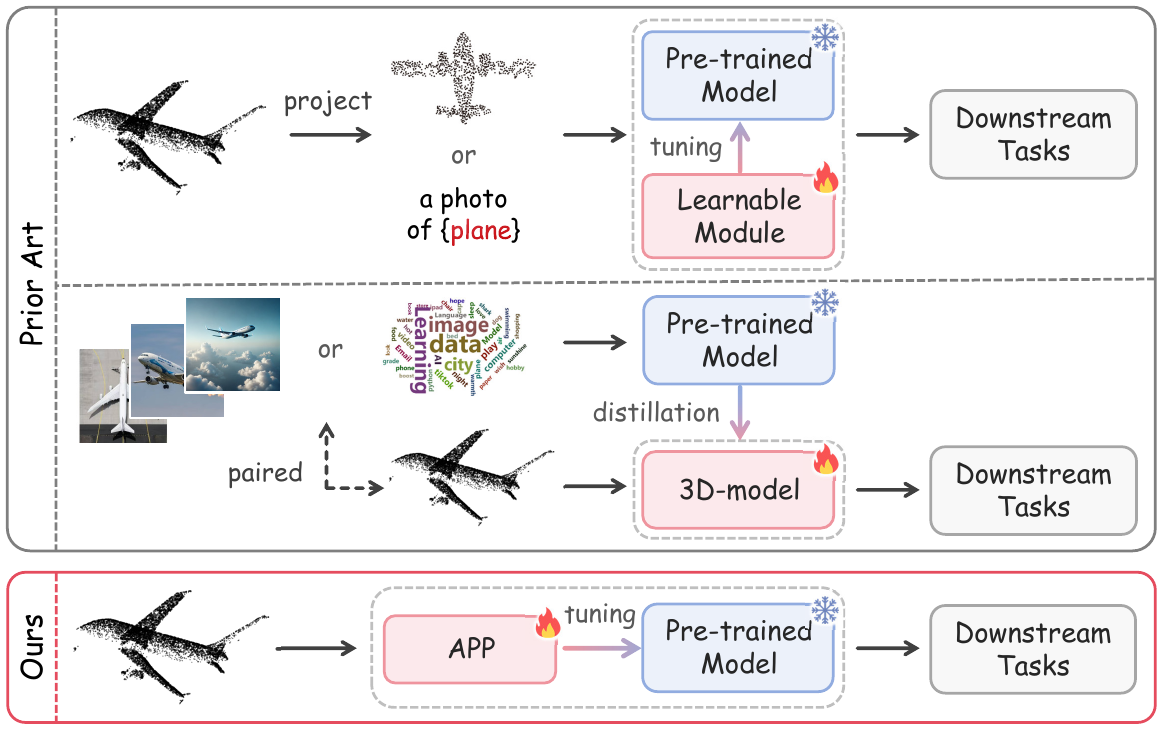}
    \caption{Comparison between existing methods and our proposed adaptive point-prompt (APP) tuning.}
    \label{fig:intro}
\end{figure}

Previous work has demonstrated the feasibility of transferring prior knowledge from heterogeneous data to 3D point cloud analysis, typically following two main routes.
1) Modality projection~\cite{Ziyi21P2P,Zhang2023Flattening-Net,Wang2024point-to-pixel,XuRS2024PointLLM, tang2025any2point} involves projecting 3D point clouds into lower-dimensional modalities, such as 1D linguistic or 2D visual representations, to leverage the pre-trained foundation models. 
However, directly projecting 3D point clouds onto 1D/2D data inevitably results in the loss of high-dimensional information.
Recently, Tang et al.~\cite{tang2025any2point} have proposed Any2Point, which virtually projects 3D coordinates to 2D (or 1D) space to utilize the position embedding of pre-trained large models.
This approach mitigates the issue of dimensional information loss by assigning positional embeddings compatible with the pre-trained model to 3D tokens. 
Nonetheless, it still relies on low-dimensional projections to exploit prior knowledge and does not fully adapt the pre-trained self-attention mechanism to the 3D domain.
2) Knowledge distillation~\cite{ZhangI2PMAE23,Liu2023OpenShape, Umam2024PartDistill,xue2023ulip,Xue2024ulip2} facilitates the training of specialized 3D models by transferring knowledge from pre-trained models trained on heterogeneous data.
However, these methods not only require training 3D models from scratch but also heavily rely on large-scale paired 2D and/or 1D-3D data. 
Their data dependencies require extensive engineering efforts, ultimately limiting their efficiency and generalization capacity.

To address these challenges, we propose a novel approach, Adaptive Point-Prompt Tuning (APPT), which directly leverages point features to adapt heterogeneous foundation models to the 3D modality, thereby optimizing the utilization of high-dimensional point cloud information while reducing computational costs.
In contrast to modality projection methods, the proposed APPT, as shown in Fig.~\ref{fig:intro}, directly processes point clouds and effectively preserves 3D information.
Specifically, APPT encodes point embeddings using farthest point sampling, k-nearest neighbors, pooling operations~\cite{QiNIPS2017pointnet2}, and local geometry aggregation to effectively handle unordered data and capture spatial features.
A linear operation is incorporated into a point embedding module to calibrate dimensionality, ensuring seamless alignment with pre-trained large models.
To enhance robustness against point permutations and capture geometric and semantic relationships between point embeddings, we exploit permutation-invariance~\cite{zaheer2017deep} for relative position injection into the token generation process.
The prompt tuning strategy~\cite{yu2023visual} is employed to adaptively fine-tune the self-attention mechanism in pre-trained models.
Notably, the prompt is a global representation generated by a point generator that shares weights with the point embedding module, followed by a pooling operation. 
As a result, the point embedding module is the only trainable component, facilitating the adaptation of pre-trained models from source modalities without the need to train an entire 3D network, thereby significantly enhancing computational efficiency.
By integrating point cloud information with the heterogeneous semantic priors from pre-trained models, APPT effectively addresses a variety of downstream 3D point cloud analysis tasks.
Extensive experiments on benchmark datasets demonstrate that the proposed APPT consistently surpasses the existing methods across various downstream tasks.
In summary, our main contributions are as follows.
\begin{itemize}
    \item We investigate the potential of pre-trained models on heterogeneous data for 3D point cloud analysis without dimension reduction and propose the APPT method to effectively leverage such models. 
    Our method demonstrates that rich 2D or 1D priors can offer valuable knowledge for the 3D domain, and with minimal fine-tuning, it can outperform models trained exclusively on 3D data.
    \item We propose a position injector (PosIn) that encodes position information with negligible training parameters.
    The concept of permutation-invariant features is introduced to identify an embedding centroid, ensuring invariance across tokens and allowing the model to remain unaffected by the order of points and tokens. 
    PosIn directs the model focus on underlying relationships and dependencies, rather than the order of points, thereby enhancing the applicability of pre-trained models.
    \item We propose a novel point-prompt generator that shares weights with the point embedding module and includes a permutation-invariant operation for obtaining order-independent global representations. 
    This generator enables direct fine-tuning of heterogeneous pre-trained models for point cloud analysis, eliminating the need for lossy mappings or time-consuming training.
    \item The proposed APPT outperforms the existing methods, as demonstrated through extensive experiments on a variety of 3D downstream tasks. 
    These experiments utilize a range of pre-trained large models, including both linguistic and visual models, consistently achieving superior performance while fine-tuning only 3.8\% of the parameters.
\end{itemize}

A preliminary version of this work has been published in~\cite{Li2024APF}. 
This paper has four major improvements. 
First, we enhance the point embedding module by incorporating local geometry aggregation and linear operations, instead of using the entire PointNet or PointMLP. This approach better handles unordered data, captures richer contextual information, and further reduces computational overhead.   
Second, we improve the fine-tuning strategy with the prompt generator, making it a plug-and-play module compatible with various pre-trained models.  
Third, we replace the sequencing operation in \cite{Li2024APF} with a position injector that has permutation-invariance property across point embeddings to enhance the feature representation of point clouds and mitigate the impact of irrelevant location information.
The ablation study demonstrates that the simple yet effective modules in APPT lead to significant improvements.
Finally, we extend the foundation model from a 2D-only model to various pre-trained foundation models, including visual, textual, and audio models, to demonstrate the effectiveness and generalization of the proposed method across diverse pre-trained knowledge sources. 
The proposed APPT consistently outperforms existing state-of-the-art methods.

\section{Related work}
\label{sec:related_work}

\subsection{MLP/CNN-based 3D Specialized Model}
Since the introduction of PointNet~\cite{qi2017pointnet}, deep learning-based approaches for point cloud processing have experienced rapid development in recent years.
These methods can be categorized into three groups based on the representations of point clouds: voxel-based~\cite{Liu2019pointvoxel,Shi2020CVPR}, projection-based~\cite{ranH2022surface,li2023bevdepth}, and point-based~\cite{guo2020deep,qianG2022pointnext}.  
Voxel-based methods entail the voxelization of input points into regular voxels, utilizing CNNs for subsequent processing. 
However, these methods tend to incur substantial memory consumption and slower runtime, particularly when a finer-grained representation is required~\cite{guo2020deep}.
Projection-based methods involve converting point clouds into dense 2D grids, which are then treated as a regular image.
This transformation enables the application of classical image-processing techniques to tackle challenges in point cloud analysis.
However, these methods heavily rely on projection and back-projection processes, presenting challenges, particularly in urban scenes with diverse scales in different directions.
In contrast, point-based methods, directly applied to 3D point clouds, are the most widely adopted. 
Such methods commonly employ shared multi-layer perceptrons or incorporate sophisticated convolution operators~\cite{qi2017pointnet,QiNIPS2017pointnet2,Wang2019Dynamic,ThomasH2019KPConv}. 
In recent years, hybrid methods such as PVCNN~\cite{Liu2019pointvoxel} and PV-RCNN~\cite{Shi2020CVPR}, which combine the strengths of diverse techniques, have achieved notable advancements.

\subsection{Self-Attention-based Specialized 3D Model}
Self-attention operations~\cite{vaswani2017attention} have been adopted for point cloud processing in several studies~\cite{zhao2021point,guoMH2021pct,choe2022pointmixer}. 
The point Transformer~\cite{zhao2021point} and point cloud Transformer (PCT)~\cite{guoMH2021pct} have introduced self-attention networks~\cite{vaswani2017attention} to improve the capture of local context within the point clouds.
Afterward, a plethora of methods based on the self-attention architecture have been proposed, which can be categorized into point-based~\cite{choe2022pointmixer,wu2022ptv2, duan2023condaformer, han2024mamba3d, wu2024ptv3}, heterogeneous auxiliary information-based~\cite{wang2022multimodal,ren2024pointofview}, and homogeneous auxiliary information-based~\cite{li2023ashapeformer, Zheng2024Diffusion,tang2024PointPEFT} methods.
Point-based methods structure point clouds by sorting them according to specific patterns, transforming unstructured, irregular point clouds into manageable sequences while preserving spatial proximity. 
This approach emulates token sequences in NLP, allowing the use of the self-attention mechanism. 
Heterogeneous auxiliary information-based methods integrate supplementary data from diverse sources (e.g., images, semantic labels) to enhance the understanding and performance of 3D point cloud tasks through multi-modal fusion and cross-modal learning techniques. 
For example, tokenFusion~\cite{wang2022multimodal} initially fuses tokens from point clouds and images, subsequently forwarding the fused tokens to a shared Transformer network, allowing the learning of correlations among multimodal features.
However, these methods suffer from high memory consumption and computational complexity~\cite{wu2022ptv2}, as they require training the entire network from scratch.
Homogeneous auxiliary information-based methods introduce 3D pre-trained models.
By fine-tuning existing pre-trained models, their performance on 3D-related tasks can be significantly improved, while computational costs can be effectively reduced.
For example, Point-Bert~\cite{yu2022point}, Point-MAE~\cite{pangYT22PointMAE}, and PointM2AE~\cite{NEURIPS2022PointM2AE} integrate masking techniques with pre-trained 3D models, enhancing the ’s generalization of models to unseen data while requiring less task-specific training.
However, compared to image data, point cloud data is more difficult to acquire, and the capability of 3D pre-trained models is relatively weaker.

\subsection{Point Cloud Analysis with 2D Foundation Model}
Leveraging knowledge from 2D to 3D seeks to strengthen the 3D understanding and improve the accuracy of 3D downstream tasks by utilizing the rich contextual information and prior knowledge embedded in pre-trained 2D models.
Most current research~\cite{Ziyi21P2P, Wang2024point-to-pixel, Zhang2023Flattening-Net,zhang2022pointclip,zhu2023pointclipv2,wei2020view-Gcn} relies on 3D-to-2D projection.
In this approach, the tokens derived from the 3D point cloud data are projected onto 2D planes, after which an existing 2D pre-trained model is employed to efficiently process the projected tokens.
While this method has proven effective, projecting 3D data to 2D introduces several challenges, such as the loss of 3D spatial information, limited handling of complex geometries, and dependency on projection angles~\cite{yang2020predicting,tang2025any2point}, to name a few.
To address these issues, several studies focus on minimizing the information loss from high-dimensional to low-dimensional representations.
For example, Any2Point~\cite{tang2025any2point} proposes a virtual projection technique to map point clouds onto 1D or 2D planes.
Nevertheless, these methods still cannot directly process 3D data.
Cross-modality knowledge distillation methods~\cite{yu2022data, xue2023ulip, dong2023act, Xue2024ulip2} typically transfer the knowledge learned by a 2D model to a smaller 3D network, enabling data-efficient training while being 3D-specific. 
The 3D model benefits from the rich prior knowledge acquired by the 2D/1D model. 
For example, ACT~\cite{dong2023act} employs pre-trained visual or language models to assist in 3D representation learning, serving as a cross-modal teacher, which enables the student model for point clouds to be trained with enhanced representational capacity.
ULIP~\cite{xue2023ulip} and ULIP-2~\cite{Xue2024ulip2} leverage the vision-language models pre-trained on large-scale image-text pairs, aligning the feature space of a point cloud encoder with the pre-aligned vision/language feature space.
However, the dependence on paired 1D/2D-3D data limits the flexibility of these methods.

\section{Adaptive Point-Prompt Tuning}

\begin{figure*}[t]
        \centering
        \includegraphics[width=\linewidth]{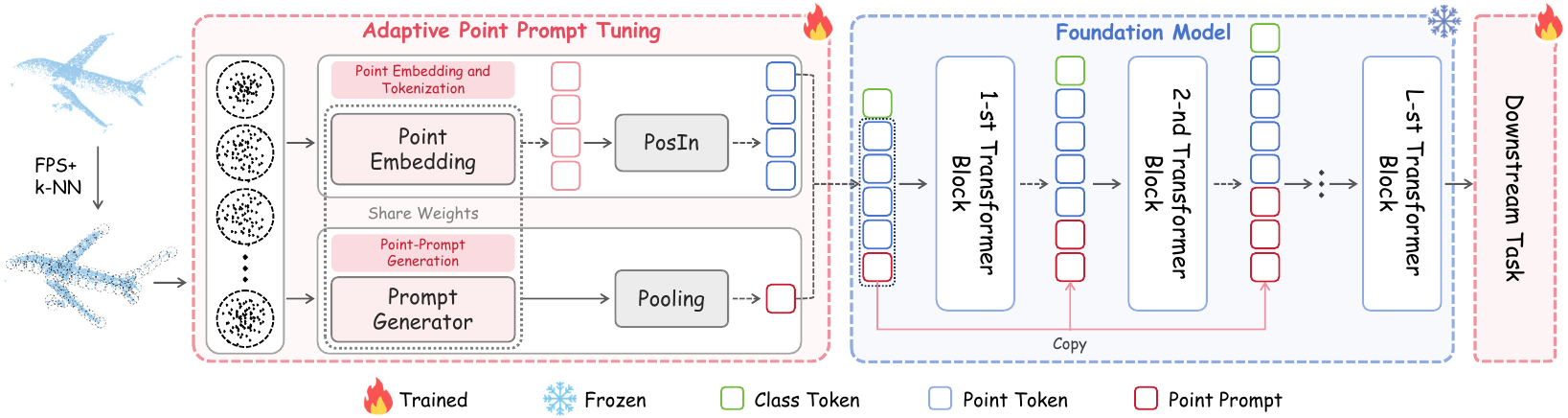}
        \caption{The structure of our proposed adaptive point-prompt tuning.}
        \label{fig:pipeline}
\end{figure*}

We propose adaptive point-prompt tuning (APPT), a method to adapt large-scale Transformer-based models pre-trained on heterogeneous data, for downstream tasks in the 3D point cloud modality.
The raw point cloud grouping input into APPT and the Transformer encoder structure employed in our approach are first overviewed in Sec.~\ref{sec:pre}.
APPT encodes the input point groups into point tokens using the point embedding and position injector modules in Sec.~\ref{sec:pnt_emb}, ensuring that the dimensionality and position information of the point tokens match those of the input tokens in the pre-trained Transformer models.
Subsequently, the prior self-attention mechanism of the foundation model is adapted by injecting a point-prompt, generated by a learnable prompt generator as described in Sec.~\ref{sec:pp-tuning}, while keeping the backbone frozen during the downstream training phase.
The token embedding module and prompt generator share knowledge to ensure consistent feature representation and reduce the number of trainable parameters. 
To enhance the structural knowledge of the point token encoding and allow for more effective information flow, APPT propagates the features encoded by each block to the next, as detailed in Sec.~\ref{sec:fine-tuning}.       
This contrasts with existing fine-tuning strategies, such as VPT-shallow and VPT-deep~\cite{jia2022visual}, where trainable prompts are inserted only into the first or each Transformer block without being passed to subsequent blocks.
The overall pipeline of APPT is illustrated in Fig.~\ref{fig:pipeline}.
Sec.~\ref{sec:rat} explains the rationale behind the proposed APPT, demonstrating its effectiveness in capturing spatial structure and global features from 3D data to provide valuable information for fine-tuning pre-trained models.

\subsection{Preliminaries}
\label{sec:pre}
\noindent\textbf{Raw Point Grouping.}
Given the input point clouds $\mathcal{P} \in \mathbb{R}^{N \times {(d'+C)}}$, where $N$ represents the number of unordered points, denoted as $\mathcal{P}=\left[x^P_1, x^P_2, \cdots, x^P_N\right]$ and $x^P_i \in \mathbb{R}^{d'+C}$ with $d'$-dim coordinates and $C$-dim point feature, we first employ iterative farthest point sampling (FPS) to sample a subset of points $\mathcal{P}_s = \left[x^P_1, x^P_2, \cdots, x^P_{N_s} \right] \in \mathbb{R}^{N_s \times (d'+C)}$.
Subsequently, the $k$-nearest neighbors $\mathcal{P}_g = \left[ \left\{x^P_{1,j} \right\}_{j=1}^k,\left\{x^P_{2,j} \right\}_{j=1}^k, \cdots, \left\{x^P_{N_s,j} \right\}_{j=1}^k \right] \in \mathbb{R}^{N_s \times k \times (d'+C)}$ for each point are identified, wherein each group $\left\{x^P_{i,j} \right\}_{j=1}^k$ within $\mathcal{P}_g$ corresponds to a local region around the centroid point $x^P_{i}$, and $k$ represents the number of points adjacent to the $N_s$ centroid points.
Following this, embedding $\mathcal{P}_g$ becomes necessary to leverage the heterogeneous priors embedded in pre-trained models.

\noindent\textbf{Transformer Encoder.} 
\label{sec:3.1}
The transformer~\cite{vaswani2017attention} encoder comprises an embedding layer and multiple transformer blocks.
For a non-point cloud input $x^H$, which can be a sentence~\cite{devlin2018bert}, an image~\cite{Alexey2021vit} or speech~\cite{chen2021developing}, the model first partitions $x^H$ into $m$ patches, forming a set $\{x^H_i\}_{i=1}^m$.
These patches are then embedded into sequences of $d^H$-dimensional vectors, denoted as $\mathcal{E}^H_0=\texttt{Embed} \left( \left[e^H_1, e^H_2, \cdots, e^H_m \right] \right)$, where $\mathcal{E}^H_0\in \mathbb{R}^{m \times d}$. 
$\mathcal{E}^H_0$ is subsequently fed into $L$ blocks $\{\phi^{(l)}\}_{i=1}^L$ within the transformer model.
We use the superscript $(l)$ to denote the index of the block.
Formally, this procedural description can be mathematically expressed as:
\begin{align}
        e_i^{H,(0)} &= \texttt{Embed}\left(x^H_i \right)+e_i, \label{eq:img_emd1}\\
        \left[e_\text{cls}^{H,(l)}, \mathcal{E}^{H,(l)} \right] &= 
        \phi^{(l)}\left( \left[e_\text{cls}^{H,(l-1)}, \mathcal{E}^{H,(l-1)} \right] \right)   \label{eq:img_emd2}  
\end{align}  
where $e_i^{H,(0)}\in \mathcal{R}^d$ and $e_i \in \mathcal{R}^d$ denote the input path embedding and positional embedding, respectively.
$\mathcal{E}^{H,(l)} = [ e_1^{H, (l)}, e_2^{H, (l)}, \cdots, e_m^{H, (l)}]$.
$e_\text{cls}^{H,(l)}$ is an additional learnable token for classification. 
$\phi^{(l)}$ is composed of multi-head self-attention ($\texttt{MHSA}$), a MLP layer ($\texttt{MLP}$) with layer normalization ($\texttt{LN}$)~\cite{ba2016layer}, and residual connection~\cite{he2016deep}. 
Specifically, $\phi^{(l)}$ is composed by:
\begin{equation} \label{eq:block}
    \begin{cases}
    \tilde{e}_i^{H,(l)} =  \texttt{MHSA}^l\left( e_i^{H,(l-1)} \right) + e_i^{H,(l-1)}\\  
    e_i^{H,(l)} = \texttt{MLP}^l \left( \texttt{LN}\left(\tilde{e}_i^{H,(l)} \right) \right) + \tilde{e}_i^{H,(l)}
    \end{cases}.
\end{equation}
A single self-attention within $\texttt{MHSA}^l$ is calculated by softmax-weighted interactions among the input query, key, and value tokens obtained by three different learnable linear projection weights. 
Finally, the class prediction is achieved by a linear classification head.

\subsection{Point Embedding and Tokenization}
\label{sec:pnt_emb}
Point embedding converts the grouped raw points into a structured and representative embedding, enhancing their utilization and alignment with the input dimensionality of the foundation model, and thereby facilitating the use of its prior knowledge.
We implement a lightweight network ($\texttt{Point\_Embed}$) to obtain the point embedding:
\begin{equation}
\label{eq:pnt_emb}
    \hat{e}^{P}_i = \texttt{Point\_Embed}\left( \mathcal{X}_i^{P} \right),
\end{equation}
where $\texttt{Point\_Embed}$ can take various forms that incorporate local geometry aggregation operations, such as PointNet~\cite{qi2017pointnet}, PointMLP~\cite{maXQ22PointMLP}, PointPN~\cite{Zhang2023ParameterIN}, to name a few.
The input point $x_i^{P}$ is from $\mathcal{P}_g$.
We use $\mathcal{X}_i^{P}$ to represent the set of $k$ neighboring points $\left\{x^P_{i,j} \right\}_{j=1}^k$ around $x^P_{i}$ for simplicity.
To seamlessly integrate with the pre-trained foundation model, the dimensionality of point embedding should align with the 2D or 1D embedding in Eq.~(\ref{eq:img_emd1}).
Specifically, $\hat{e}^{P}_i \in \mathbb{R}^d$.
Eventually, the embedding representation of an input point cloud $\mathcal{P}$ for feeding into pre-trained foundation model is $\hat{\mathcal{E}}^{P} = \left[ \hat{e}^{P}_1, \hat{e}^{P}_2, \cdots, \hat{e}^{P}_{N_s} \right] $.

The inherent unordered nature of point clouds is one of their most significant properties~\cite{qi2017pointnet}, distinguishing 3D data from pixel arrays in visual data and sequences in linguistic data.
Merely aligning the dimensionality of embeddings is insufficient to fully leverage the attention-related priors of a pre-trained transformer. 
Based on the positional encoding in the Transformer~\cite{vaswani2017attention}, we propose the position injector. 
It injects sufficient positional information from the source modality into 3D tokens, enabling more effective collaboration with the frozen transformer.
We use average pooling, $\texttt{avgP}: \mathbb{R}^ {N_s \times d} \rightarrow \mathbb{R}^{1 \times d}$, to obtain a global embedding $e_{g}$ that represents the centroid of the input: 
\begin{equation}\label{eq:PIF}
   e_{g} = \texttt{avgP}\left( \hat{\mathcal{E}}^{P} \right).
\end{equation}
Then, the input point token $e^{P,(0)}_i$ fed into the transformer blocks is obtained by a linear combination of the relative position and the point embedding:
\begin{equation}\label{eq:ini_posIn}
    e^{P,(0)}_i = a \cdot (\hat{e}^{P}_i- e_{g})+ b\cdot \hat{e}^{P}_i,
\end{equation}
where $a$ and $b$ are learnable parameters. 
They can be replaced by a 1D convolution kernel, allowing this module to be seamlessly integrated into an existing model as a standalone layer.
Therefore, Eq.~(\ref{eq:ini_posIn}) can be changed into the following form:
\begin{equation}\label{eq:posIn}
     e^{P,(0)}_i= \texttt{Conv1D}\left( \texttt{Concat}\{ \hat{e}^{P}_i- e_{g}, \hat{e}^{P}_i\} \right),
\end{equation}
where $\texttt{Conv1D}$ denotes 1D convolution operation, and $\texttt{Concat}\{\}$ represents the concatenation of the inputs. 
Since it contains only two training parameters (without using a bias term), the increase in the total number of training parameters is negligible.
The structure of this position injector (PosIn) is shown in Fig.~\ref{fig:PosIn}.

\begin{figure}[t]
        \centering
        \includegraphics[width=0.8\linewidth]{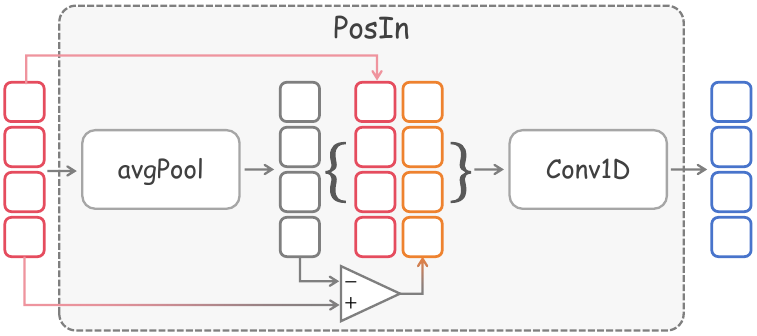}
        \caption{The structure of our proposed position injector (PosIn).
        We encode the location information of point tokens by embedding their relative positional differences.}
        \label{fig:PosIn}
\end{figure}

\subsection{Point-Prompt Generation}
\label{sec:pp-tuning}
Training transformer-based architectures from scratch generally requires larger datasets compared to CNN-based ones~\cite{Yuan21Tokens2Token}.
Compared to text and image data, the availability of 3D data is relatively constrained, leading to challenges such as overfitting and suboptimal utilization of the full potential of transformer-based models. 
This paper investigates PEFT technology to alleviate overfitting and improve model generalization for 3D models. 
PEFT involves the freezing of the pre-trained backbone that is previously trained on an extensive dataset, while introducing a limited number of learnable parameters to adapt to the new dataset.
This new dataset can be data-rich~\cite{jia2022visual,bahng2022exploring}, few-shot~\cite{LeeDJ2023read}, or long-tailed~\cite{DongB2022lpt,shi2024LIFT,li2024GNMPT}, as PEFT equips the model with knowledgeable priors.

Prompt tuning~\cite{Xiang2021PrefixTuning,jia2022visual} is one of the most effective and widely used PEFT methods.
It appends trainable prompts to the tokens in Eq.~(\ref{eq:block}) to fine-tune self-attention for different tasks and has been empirically validated for its effectiveness in handling both 1D linguistic and 2D visual data.
We apply this technique to integrate heterogeneous prior attention with point tokens.
Different from prompt tuning~\cite{Xiang2021PrefixTuning} and VPT~\cite{jia2022visual}, we propose a trainable prompt generator for prompt generation.
Prompts for point clouds (point-prompts) should satisfy the following properties: 1) they are closely related to the input, 2) they capture the overall information of the input point clouds, and 3) they share the same dimensionality as the point embeddings. 
To achieve this, we adopt the same structure as token embedding and introduce a pooling operation to capture the overall features of the input, which are then used as the point-prompts.
To maintain consistency between the point tokens, which capture local features, and the point-prompt, which encodes overall features, while also reducing training parameters, we make the parameters between the point embedding module and the prompt generator module shared.
Consequently, the point-prompt $p_0$ fed into the subsequent transformer blocks is calculated as follows:
\begin{equation}\label{eq:prompt}
p_0 = \texttt{maxP}\left( \hat{\mathcal{E}}^{P} \right) + \texttt{avgP}\left( \hat{\mathcal{E}}^{P} \right),
\end{equation}
where $\texttt{maxP}$ and \texttt{avgP} refer to max pooling and average pooling, respectively. 
$p_0$ is permutation-invariant to the raw point groups, ensuring that the model remains insensitive to the order of point group arrangement (the detailed proof will be provided in Sec.~\ref{sec:rat}).
This prompt generator provides three main advantages:
1) It provides more stable global features; 
2) It eliminates redundant information; and 
3) The generated point-prompt preserves the geometric information of the input point clouds.

\subsection{Effective Fine-Tuning of Transformer Blocks}\label{sec:fine-tuning}
Given a pre-trained foundation model, the generated point-prompt is incorporated into each transformer block. 
During fine-tuning, only the task-specific prompt generator is updated, while the transformer backbone remains fixed.
The point-prompt serves two primary functions: 1) it adapts the prior self-attention mechanisms within the pre-trained transformer model; 2) it encodes the global features of the input point cloud to provide structural information - distinguishing it from existing prompt-tuning techniques.
Consequently, we concatenate the generated point-prompt to each block and retain it at the output of each transformer block, preserving the original encoded point cloud structure while maintaining the interaction between the pre-trained prior and the point-prompt. 
The point-prompted transformer blocks are formulated as:
\begin{align}
    & \quad \mathcal{E}^{P,(0)} =  \texttt{PosIn}\left(\texttt{Point\_Embed}\left(\left[\mathcal{X}^p_i \right] \right) \right) , \label{eq:p_emd1}\\
    & \left[e_\text{cls}^{P,(1)}, \mathcal{Z}^{(1)},  \mathcal{E}^{P,(1)}\right]  = \phi^{(1)}\left( \left[\textcolor{myblue}{e_\text{cls}^{P,(0)}}, \textcolor{myred}{p_0}, \textcolor{myred}{\mathcal{E}^{P,(0)}} \right] \right),\\
    & \left[e_\text{cls}^{P,(l)}, \mathcal{Z}^{(l)},   \mathcal{E}^{P,(l)} \right]  = 
    \phi^{(l)}\left( \left[ \textcolor{myblue}{e_\text{cls}^{P,(l-1)}}, \left[\textcolor{myred}{p_0},\textcolor{myblue}{\mathcal{Z}^{(l-1)}} \right], \textcolor{myblue}{\mathcal{E}^{P,(l-1)}} \right] \right),   \label{eq:p_emd2}  
\end{align}  
where $\texttt{PosIn}$ is calculated by Eqs.~\ref{eq:PIF} and \ref{eq:ini_posIn}.
$\mathcal{Z}^l \in \mathbb{R}^{l\times d}$ denotes the features generated by the $l$-th transformer block.
The colors \textcolor{myred}{red} and \textcolor{myblue}{blue} indicate intermediate variables that originate from trainable and frozen modules, respectively.

For the input token to the downstream head, Li et al.~\cite{li2024GNMPT} proposed that all learnable prompts are trained on the fine-tuning dataset, thereby incorporating newly acquired information. 
They propose the "merge prompt" strategy, which linearly combines all the learned prompts from the final block into a class token.
Inspired by this approach, in our work, both point tokens and prompts are learned from the point cloud dataset. 
We employ a pooling operation, following the Swin transformer~\cite{liu2021swin}, to integrate the newly learned knowledge into the final class token:
\begin{equation}\label{eq:cls_token}
    e_\text{cls} =   \texttt{Pool} \left( \left[e_\text{cls}^{P,(L)}, \mathcal{Z}^{(L)}, \mathcal{E}^{P,(L)} \right] \right),
\end{equation}
where $\texttt{Pool}$ adopts the sum of max and average pooling.
APPT can be beneficial for multiple 3D downstream tasks due to its minimal training cost. 
Only the prompt generator, which shares weights with the point embedding module, and the task-specific head need to be trained.
There are two main downstream tasks:

\noindent\textbf{Classification} involves labeling and categorizing the entire point cloud.
The predicted logit for each class is obtained by applying the softmax function to the output of the final linear layer:
\begin{equation}
    p_i = \dfrac{e^{w_i \cdot e_\text{cls}}}{ \sum_{j=1}^C e^{w_j \cdot e_\text{cls}} },  
\end{equation}
where $w_i$ is the weight of the classification head and $C$ is the total number of classes.
Eventually, the cross-entropy loss can be utilized to calculate the loss function.

\noindent\textbf{Segmentation} involves dividing 3D point cloud data into multiple subsets or regions with similar attributes. 
To achieve this, we utilize a U-Net-style architecture, where the APPT serves as the point encoder. 
The segmentation head concatenates the output features from the transformer blocks within the encoder, followed by deconvolutional interpolation and multiple MLP layers to enable dense prediction. 
Similar to the classification task, the softmax cross-entropy is used as the loss function.

\subsection{Rational Analysis}
\label{sec:rat}
In point cloud analysis, tasks such as classification and segmentation rely on the spatial distribution of points, rather than their order. 
We introduce the \textbf{permutation-invariant}~\cite{NIPS2017DeepSets,NIPS2019DeepSetNet} and show this property of our method.
\begin{definition} \textbf{(Permutation-invariant function.)} For a set $S = \{s_1, s_2, \cdots, s_n\}$, a function $g: \mathbb{R}^ {d_1 \times d} \rightarrow \mathbb{R}^{d_2}$ is permutation-invariant iff it satisfies 
\begin{equation}
    g(\mathbf{S}) = g(\mathbf{\sigma (S)}),
\end{equation}
for any permutation $\sigma$ (any reordering of the elements).
\end{definition}

\begin{lemma} \label{lem:max}
The max operation, \( \max: \mathbb{R}^ {d} \rightarrow \mathbb{R} \), is a permutation-invariant function.
\end{lemma}

\begin{proof}
Let $\sigma$ be an arbitrary permutation of the set $S$.
By definition, $\sigma$ is a bijective function that rearranges the elements of $s$, such that $\sigma(S) = \{s_{\sigma(1)}, s_{\sigma(2)}, \dots, s_{\sigma(n)} \}$ for $s_i \in S $.
Since $\max(S)$ selects the largest element in $S$, and the permutation $\sigma$ does not alter the set content, we have $\max(S) = \max(\sigma(S)) $.
\end{proof}

\begin{lemma} \label{lem:mean}
The mean operation, $\texttt{mean}: \mathbb{R}^ {d} \rightarrow \mathbb{R}$, is a permutation-invariant function.
\end{lemma}

\begin{proof}
Let $\sigma$ be an arbitrary permutation of the set $S$, where \( S = \{s_1, s_2, \dots, s_n\} \).
The mean of the set $\sigma(S)$ is given by
\begin{equation}
\texttt{mean} \left( \sigma(S) \right) = \frac{1}{n} \sum_{i=1}^{n} s_{\sigma(i)}.
\end{equation}
Since $\sigma$ is a bijective function, $ \sigma(S)$ contains exactly the same elements as $S$.
Furthermore, by the commutative property of addition, we can rearrange the terms in the sum without changing its value,
\begin{equation}
    \frac{1}{n} \sum_{i=1}^{n} s_{\sigma(i)} =  \frac{1}{n} \sum_{i=1}^{n} s_{i}
\end{equation}
Thus, we conclude that:
\begin{equation}
\texttt{mean}(\sigma(S)) = \texttt{mean}(S) .
\end{equation}
\end{proof}
By Lemmas~\ref{lem:max} and \ref{lem:mean}, we can deduce the following theorem regarding pooling operations.
\begin{theorem}\label{thm:pool} 
The max pooling and average pooling across channels, $\texttt{maxP}: \mathbb{R}^ {c \times d} \rightarrow \mathbb{R}^{1 \times d} $ and $\texttt{avgP}: \mathbb{R}^ {c \times d} \rightarrow \mathbb{R}^{1 \times d} $, are both permutation-invariant functions.
\end{theorem}
The global embedding $e_{g}$ (as defined in Eq.~\ref{eq:PIF} of Sec.~\ref{sec:pnt_emb}) is utilized to determine the relative position and, therefore, must remain invariant to the ordering of point embeddings.  
Similarly, the point-prompt $p_o$ (as defined in Eq.~\ref{eq:prompt} of Sec.~\ref{sec:pp-tuning}) offers a comprehensive representation of the input, while the final class token $e_\text{cls}$ (as defined in Eq.~\ref{eq:cls_token} of Sec.~\ref{sec:fine-tuning}) is the global feature that integrates both the input data and the prior knowledge from the foundation model.
Since the order of point embeddings does not reflect the spatial relationship or structure of the input point cloud, both $p_o$ and $e_\text{cls}$ should also be unaffected by the permutation of point embeddings.
Theorem~\ref{thm:pool} shows that $e_{PI}$, $p_o$ and $e_\text{cls}$ are permutation-invariant with respect to the order of point embeddings.
This property enables our proposed APPT to effectively extract spatial structure and global features from point cloud data, allowing the model to better cope with noise and sampling unevenness.

\section{Experiment}
\subsection{Datasets and Basic Settings}

\begin{table*}[t]
 \caption{Comparisons on accuracy for object classification on ScanObjectNN and ModelNet40.
The best and second-best results are highlighted in \textbf{\underline{underlined bold}} and \textbf{bold}, respectively.
The superscript * denotes results obtained using ViT-B for P2P to ensure a fair comparison.
``Aud.'' is an abbreviation for ``Audio''.
}
 \label{tab:com_cls_SONN_MN40}
 \centering  
 \renewcommand{\arraystretch}{1.1}
 \resizebox{0.8\linewidth}{!}  
{\begin{tabular} {m{2.9cm} | m{1.4cm}<{\centering} m{1.4cm}<{\centering} | m{1.4cm}<{\centering} m{1.4cm}<{\centering} m{1.4cm}<{\centering} | m{1.4cm}<{\centering}} 
  \hlinew{1pt}
  \multirow{2}{*}{Methods}  & Published & Pretrained & \multicolumn{3}{c|}{ScanObjectNN}  & \multirow{2}{*}{ModelNet40}  \\
  \cline{4-6}
                            & Year      & Modality   & OBJ-BG & OBJ-ONLY & PB-T50-RS     &  \\ 
  \hline  
  \multicolumn{7}{c}{MLP/CNN-based Model} \\
  \hline
  PointNet~\cite{qi2017pointnet} & 2017 & N/A & 73.8 & 79.2 & 68.0 & 89.2\\  %
  DGCNN~\cite{Wang2019Dynamic} & 2019 & N/A & 82.8 & 86.2 & 78.1 & 92.9\\      %
  PointMLP~\cite{maXQ22PointMLP}& 2022  & N/A & - & - & 85.2 &94.1 \\   
  Point-PN~\cite{Zhang2023ParameterIN} & 2023 & N/A & 91.0 & 90.2 & 87.1 & 93.8\\
  \hdashline
  PointNet-OcCo~\cite{wangHC2021occo}  & 2021 & 3D & - & - & 80.0 & 90.1\\    %
  DGCNN-OcCo~\cite{wangHC2021occo}  & 2021 & 3D & - & - & 83.9 & 93.0\\     %
  \hline  
  \multicolumn{7}{c}{MHSA-based Model} \\
  \hline  
  Transformer~\cite{vaswani2017attention} &2017  & N/A & 79.9 & 80.6 & 77.2 &91.4 \\  %
  PCT~\cite{guoMH2021pct} & 2021 & N/A &  - &- & -& 93.2 \\
  \hdashline
  Transformer-OcCo~\cite{wangHC2021occo} &2021 & 3D & 84.9 & 85.5 & 78.8 & 92.1\\  %
  Point-BERT~\cite{yu2022point} &2022 & 3D & 87.4 & 88.1 & 83.1  & 93.2 \\    %
  Point-MAE~\cite{pangYT22PointMAE} &2022 & 3D & 90.0 & 88.3 & 85.2 &93.8 \\     %
  Joint-MAE~\cite{GuoZQLH23JointMAE} &2023& 3D & 90.9 & 88.9 & 86.1 & 94.0\\      %
  \makecell[l]{Point-BERT \\ \: w. Point-PEFT~\cite{tang2024PointPEFT} } 
                               &2024 & 3D & - & - & 85.0  & 93.4 \\ 
  \makecell[l]{Point-BERT \\ \: w. DAPT~\cite{Xin2024DAPT} } 
                               &2024 & 3D & 91.1 & 89.7 & 85.4  & 93.6 \\                                
  \hdashline
  P2P$^*$~\cite{Ziyi21P2P}  &2022 & 2D & - & - & 84.1  &92.4  \\          %
  APF~\cite{Li2024APF} &2024 & 2D  & 89.9 & 89.0 & 87.8 & 94.2 \\ %
  Any2Point~\cite{tang2025any2point} &2024  & 2D & - & - & 87.7 & 93.2\\ %
 \rowcolor{mygray}
  APPT & Ours & 2D & \underline{\textbf{92.4}} & 90.5 & \underline{\textbf{92.6}} & 94.2\\
  \hdashline
  ACT~\cite{dong2023act} & 2023 & 3D+2D & 87.1 & 89.0 & 81.5   & 93.7\\
  ReCon~\cite{qi2023contrast} & 2023 & \makecell[c]{3D+2D+ \\ 1D (Text) }  & 90.6 & \underline{\textbf{90.7}} & 83.8   & 93.4\\
  Any2Point~\cite{tang2025any2point} &2024 & 1D (Aud.) & - & - & 87.0 & 92.7 \\ %
  Any2Point~\cite{tang2025any2point}  &2024 &  1D (Text) & - & - & \textbf{91.9} & 94.3 \\ %
  \rowcolor{mygray}
  APPT  & Ours & 1D (Aud.) & \textbf{92.3} & \underline{\textbf{90.7}} & 88.9 & \textbf{94.6} \\
  \rowcolor{mygray}
  APPT  & Ours & 1D (Text) & 91.9 & 90.2 & 91.4 & \underline{\textbf{95.1}}\\   
  \hlinew{1pt}
  \end{tabular}}
\end{table*}


\begin{table*}[tb]
\caption{Few-shot classification results on ModelNet40.} 
\label{tab:com_few_shot} 
 \centering  
 \renewcommand{\arraystretch}{1.1}
 \resizebox{0.85\linewidth}{!}  
{\begin{tabular}
{m{3cm} |  m{1.5cm}<{\centering} | m{2cm}<{\centering} m{2cm}<{\centering} | m{2cm} <{\centering}m{2cm}<{\centering}}
  \hlinew{1pt}
  \multirow{2}{*}{Methods} &\multirow{2}{*}{\makecell[c]{Pre-trained \\ Modality }}& \multicolumn{2}{c|}{5-way} & \multicolumn{2}{c}{10-way} \\ 
  \cline{3-6}  &   & 10-shot & 20-shot & 10-shot & 20-shot  \\ 
  \hline
  \multicolumn{6}{c}{MLP/CNN-based Model} \\
  \hline
  PointNet~\cite{qi2017pointnet} & N/A &52.0\ $\pm$\ 3.8 &57.8\ $\pm$\ 4.9 & 46.6\ $\pm$\ 4.3 & 35.2\ $\pm$\ 4.8 \\
  PointNet-OcCo~\cite{wangHC2021occo} & 3D &89.7\ $\pm$\ 1.9 &92.4\ $\pm$\ 1.6 & 83.9\ $\pm$\ 1.8 & 89.7\ $\pm$\ \textbf{1.5} \\  
  \makecell[l]{PointNet \\ \: w. CrossPoint~\cite{AfhamM22CrossPoint} } 
                                      & 2D &90.9\ $\pm$\ 4.8 &93.5\ $\pm$\ 4.4 &84.6\ $\pm$\ 4.7 &90.2\ $\pm$\ 2.2 \\
  DGCNN~\cite{Wang2019Dynamic}  & N/A  &31.6\ $\pm$\ 2.8 &40.8\ $\pm$\ 4.6 &19.9\ $\pm$\ 2.1 &16.9\ $\pm$\ 1.5 \\ 
  DGCNN-OcCo~\cite{wangHC2021occo} & 3D  &90.6\ $\pm$\ 2.8 &92.5\ $\pm$\ 1.9 &82.9\ $\pm$\ 1.3 &86.5\ $\pm$\ 2.2 \\ 
  \makecell[l]{DGCNN \\ \: w. CrossPoint~\cite{AfhamM22CrossPoint} } 
                                & 2D  &92.5\ $\pm$\ 3.0 &94.9\ $\pm$\ 2.1 &83.6\ $\pm$\ 5.3 &87.9\ $\pm$\ 4.2 \\
  \hline 
  \multicolumn{6}{c}{MHSA-based Model} \\
  \hline
  Transformer~\cite{vaswani2017attention} & N/A &87.8\ $\pm$\ 5.2 &93.3\ $\pm$\ 4.3 &84.6\ $\pm$\ 5.5 &89.4\ $\pm$\ 6.3\\
  Transformer-OcCo~\cite{wangHC2021occo}  & 3D &94.0\ $\pm$\ 3.6 &95.9\ $\pm$\ 2.3 &89.4\ $\pm$\ 5.1 &92.4\ $\pm$\ 4.6\\
  Point-BERT~\cite{yu2022point}   & 3D & 94.6\ $\pm$\ 3.1 &96.3\ $\pm$\ 2.7 &91.0\ $\pm$\ 5.4 &92.7\ $\pm$\ 5.1\\
  Point-MAE~\cite{pangYT22PointMAE}  & 3D &96.3\ $\pm$\ 2.5 & 97.8\ $\pm$\ 1.8 & 92.6\ $\pm$\ 4.1 & 95.0\ $\pm$\ 3.0\\
  Joint-MAE~\cite{GuoZQLH23JointMAE}  & 3D &96.7\ $\pm$\ 2.2 & 97.9\ $\pm$\ 1.8 & 92.6\ $\pm$\ 3.7 & 95.1\ $\pm$\ 2.6\\
  \makecell[l]{Point-BERT \\ \: w. DAPT~\cite{Xin2024DAPT} }
                                    & 3D &95.8 $\pm$\ 2.1 & 97.3\ $\pm$\ 1.3 & 92.2\ $\pm$\ 4.3 & 94.2\ $\pm$\ 3.4\\
  APF~\cite{li2023ashapeformer}   & 2D & \textbf{96.9}\ $\pm$\ 1.8 & 98.1 $\pm$\ 1.8 & \textbf{92.6}\ $\pm$\ 2.4 & \underline{\textbf{95.7}} \ $\pm$\ 1.6 \\    
  \hdashline
  \rowcolor{mygray}
  APPT (ours)  & 2D & \underline{\textbf{97.0}}\ $\pm$\ \underline{\textbf{1.0}} & \underline{\textbf{99.1}} $\pm$\ \underline{\textbf{0.9}} & \underline{\textbf{92.7}}\ $\pm$\ \underline{\textbf{0.8}} & \textbf{95.3} \ $\pm$\ 2.3 \\  
  \rowcolor{mygray}
  APPT (ours)  & 1D (Text) & 96.5\ $\pm$\ 2.0 & 99.0 $\pm$\ 1.0 & 91.5\ $\pm$\ 2.5 & 95.1 \ $\pm$\ 2.1 \\
  \rowcolor{mygray}
  APPT (ours)  & 1D (Aud.) & 96.5\ $\pm$\ 1.5 & \underline{\textbf{99.1}} $\pm$\ \underline{\textbf{0.9}} & 91.4\ $\pm$\ 1.6 & 94.9 \ $\pm$\ 1.9 \\
  \hlinew{1pt} 
  \end{tabular}}
\end{table*}

\noindent\textbf{Datasets.}
We conduct object classification tasks using the widely used benchmarks, ScanObjectNN~\cite{uy2019revisiting} and ModelNet40~\cite{wuZR20153d}.
ScanObjectNN is a challenging dataset with inherent scan noise and occlusion, consisting of 15,000 scanned objects across 15 distinct classes, sampled from the real world. 
In line with prior work, we conduct experiments on three variants: OBJ-BG, OBJ-ONLY, and PB-T50-RS. 
ModelNet40 contains 12,311 CAD models across 40 object categories. 
We follow the official data split, with 9,843 objects for training and 2,468 for evaluation, ensuring a fair comparison.
For part segmentation, we utilize ShapeNetPart~\cite{yi2016scalable}, a meticulously annotated 3D dataset derived from ShapeNet. 
ShapeNetPart encompasses 16 distinct shape categories, each annotated at the part level across 50 classes. Notably, each category is further delineated into 2 to 6 unique parts, providing granularity and specificity essential for detailed segmentation analysis.

\begin{table*}[t]
\caption{Part segmentation results on ShapeNetPart. $\text{mIoU}_C$ (\%) is the mean of class IoU. $\text{mIoU}_I$ (\%) is the mean of instance IoU. ``Trans.'' abbreviates for Transformer.}  \label{tab:com_seg}
\centering  
\renewcommand{\arraystretch}{1.2}
\resizebox{\linewidth}{!}  
{\begin{tabular}
  {m{2.1cm} | m{0.6cm}<{\centering}   m{0.6cm}<{\centering} m{0.5cm}<{\centering} m{0.5cm}<{\centering} m{0.5cm}<{\centering} m{0.5cm}<{\centering} m{0.5cm}<{\centering} m{0.5cm}<{\centering} m{0.5cm}<{\centering} m{0.5cm}<{\centering} m{0.5cm}<{\centering} m{0.5cm}<{\centering} m{0.5cm}<{\centering} m{0.5cm}<{\centering} m{0.5cm}<{\centering} m{0.5cm}<{\centering} m{0.5cm}<{\centering} m{0.5cm}<{\centering} }
  \hlinew{1pt}
  Methods & $\text{mIoU}_C$ & $\text{mIoU}_I$ & aero-plane &  bag & cap & car & chair & ear-phone  & guitar & knife & lamp & laptop & motor-bike & mug & pistol & rocket & skate-board & table \\ 
  \hline  
  \multicolumn{19}{c}{MLP/CNN-based Model} \\
  \hline
  PointNet~\cite{qi2017pointnet} & 80.4 & 83.7 & 83.4 & 78.7 & 82.5 & 74.9 & 89.6 & 73.0 & 91.5 & 85.9 & 80.8 & 95.3 & 65.2 & 93.0 & 81.2 & 57.9 & 72.8 & 80.6 \\
  PointNet++~\cite{QiNIPS2017pointnet2} & 81.9 & 85.1 & 82.4 & 79.0 & 87.7 & 77.3 & 90.8 & 71.8 & 91.0 & 85.9 & 83.7 & 95.3 & 71.6 & 94.1 & 81.3 & 58.7 & 76.4 & 82.6 \\ 
  DGCNN~\cite{Wang2019Dynamic} & 82.3 & 85.2 & 84.0 & 83.4 & 86.7 & 77.8 & 90.6 & 74.7 & 91.2 & 87.5 & 82.8 & 95.7 & 66.3 & 94.9 & 81.1 & 63.5 & 74.5 & 82.6 \\
  KPConv~\cite{ThomasH2019KPConv} & 85.1 & \textbf{86.4} & 84.6 & 86.3 & 87.2 & 81.1 & 91.1 & 77.8 & 92.6 & 88.4 & 82.7 & 96.2 & 78.1 & 95.8 & 85.4 & 69.0 & 82.0 & 83.6\\
  PAConv~\cite{xuM2021paconv}  & 84.6 & 86.1 & - & - & - & - & - & - & - & - & - & - & - & - & - & - & - & - \\
  PointMLP~\cite{maXQ22PointMLP} & 84.6 & 86.1 & 83.5 & 83.4 & 87.5 & 80.54 & 90.3 & 78.2 & 92.2 & 88.1 & 82.6 & 96.2 & 77.5 & 95.8 & 85.4 & 64.6 & 83.3 & 84.3 \\
  \hline
  \multicolumn{19}{c}{MHSA-based Model} \\
  \hline
  Trans.~\cite{vaswani2017attention} & 83.4 & 85.1 & 82.9 & 85.4 & 87.7 & 78.8 & 90.5 & 80.8 & 91.1 & 87.7 & 85.3 & 95.6 & 73.9 & 94.9 & 83.5 & 61.2&  74.9 & 80.6 \\
  Point Trans.~\cite{zhao2021point} & 83.7 & \underline{\textbf{86.6}} & - & - & - & - & - & - & - & - & - & - & - & - & - & - & - \\
  PCT~\cite{guoMH2021pct} & - &\textbf{86.4} & 85.0 & 82.4 & 89.0 & 81.2 & 91.9 & 71.5 & 91.3 & 88.1 & 86.3 & 95.8 & 64.6 & 95.8 & 83.6 & 62.2 & 77.6 & 83.7 \\
  Trans.-OcCo~\cite{wangHC2021occo} & 83.4 & 85.1 & 83.3 & 85.2 & 88.3 & 79.9 & 90.7 & 74.1 & 91.9 & 87.6 & 84.7 & 95.4 & 75.5 & 94.4 & 84.1 & 63.1 & 75.7 & 80.8  \\
  Point-BERT~\cite{yu2022point} & \textbf{84.1} & 85.6 & 84.3 & 84.8 & 88.0 & 79.8 & 91.0 & 81.7 & 91.6 & 87.9 & 85.2 & 95.6 & 75.6 & 94.7 & 84.3 & 63.4 & 76.3 & 81.5 \\
  Point-MAE~\cite{pangYT22PointMAE} & - & 86.1 & 84.3& 85.0 & 88.3 & 80.5 & 91.3 & 78.5 & 92.1 &  87.4 & 96.1 & 96.1 & 75.2 & 94.6 & 84.7 & 63.5 & 77.1 & 82.4 \\
  P2P$^*$~\cite{Ziyi21P2P} & 82.5 & 85.7 & 83.2 & 84.1 & 85.9 & 78.0 & 91.0 & 80.2 & 91.7 & 87.2 & 85.4 & 95.4 & 69.6 & 93.5 & 79.4 & 57.0 & 73.0 & 83.6 \\
  Joint-MAE~\cite{GuoZQLH23JointMAE} & \underline{\textbf{85.4}} & 86.3 & - & - & - & - & - & - & - & - & - & - & - & - & - & - & - & - \\
  \makecell[l]{Point-BERT~\cite{yu2022point} \\ \: w. DAPT~\cite{Xin2024DAPT} } 
                         & 83.8 & 85.5 & - & - & - & - & - & - & - & - & - & - & - & - & - & - & - \\
  APF~\cite{Li2024APF}  & 83.4 & 86.1 & 83.6 & 84.8 & 85.4 & 79.8 & 91.3 & 77.0 & 91.4 & 88.4 & 84.4 & 95.5 & 76.3 & 95.3 & 82.5 & 59.5 & 76.1 & 83.5 \\
  \hdashline
  \rowcolor{mygray}
  APPT (ours) & 84.0 & 85.9 & 83.5 & 85.0 & 86.7 & 79.8 & 91.9 & 79.6 & 91.9 & 87.9 & 83.7 & 96.1 & 76.2 & 95.8 & 82.2 & 65.1 & 76.4 & 82.8 \\
  \hlinew{1pt}
\end{tabular}}
\end{table*}

\noindent\textbf{Implementation Details.}
We follow the settings in~\cite{Ziyi21P2P} and~\cite{GuoZQLH23JointMAE}, using the AdamW optimizer in combination with the Cosine annealing scheduler.
The learning rate is initialized at $5\times 10^{-4}$, with a weight decay of $5 \times 10^{-2}$. 
For the point embedding module, we explore an architecture based on Point-PN~\cite{Zhang2023ParameterIN}.
The output dimensionality of the point embedding module is set to 768 to match the input feature channels of the Transformer architecture.
In comparison experiments, the ViT-Base version (ViT-B)~\cite{Dosovitskiy21vit} pre-trained on imageNet21K~\cite{russakovsky2015imagenet} is utilized as the pre-trained 2D model, which is widely adopted in previous work~\cite{Ziyi21P2P, li2023ashapeformer}. 
For the 1D model, we leverage ImageBind audio encoder~\cite{girdhar2023imagebind} for the audio prior and CLIP text encoder~\cite{radford2021clip} for the language prior, respectively.
For few-shot classification and part segmentation, we conduct experiments using the 2D pre-trained ViT-B.
In the ablation study, we further investigate the impact of various pre-trained models to rigorously validate the effectiveness of our proposed APPT framework. 
Specifically, we employ DINOv2~\cite{Oquab2023DINOv2LR} and DeiT~\cite{Touvron2020TrainingDI} as alternative visual priors, and RoBERTa~\cite{Liu2019RoBERTaAR} as the linguistic prior, to assess the robustness and generalizability of our proposed APPT across different pre-trained architectures.

\noindent\textbf{Comparison Methods.}
We compare our APPT with two primary categories of methods. 
The first category consists methods based on multilayer perceptron (MLP) or convolutional neural network (CNN), including foundational works such as PointNet~\cite{qi2017pointnet}, DGCNN~\cite{Wang2019Dynamic}, as well as more recent advancements like PointMLP~\cite{maXQ22PointMLP} and Point-PN~\cite{Zhang2023ParameterIN}. 
Additionally, we evaluate against 3D pre-trained models, including OcCo~\cite{wangHC2021occo}, which integrates PointNet and DGCNN.
The second category comprises methods leveraging multi-head self-attention (MHSA), including the basic Transformer~\cite{vaswani2017attention} and its adaptations for point cloud data, such as Point Cloud Transformer (PCT)~\cite{guoMH2021pct}.
Additionally, we compare our approach with methods that utilize 3D pre-trained models, such as Transformer-OcCo\cite{wangHC2021occo}, Point-BERT~\cite{yu2022point}, Point-MAE~\cite{pangYT22PointMAE}, Joint-MAE~\cite{GuoZQLH23JointMAE}, and fine-tuned Point-BERT with Point-PEFT~\cite{tang2024PointPEFT}.
To ensure a comprehensive evaluation, we also include methods employing 2D pre-trained models, such as P2P~\cite{Ziyi21P2P}, our conference version APF~\cite{Li2024APF}, and Any2Point~\cite{tang2025any2point}. 
Furthermore, we extend our comparisons to pre-trained models from other modalities, including ACT~\cite{dong2023act}, ReCon~\cite{qi2023contrast}, and Any2Point~\cite{tang2025any2point}, which integrate audio and text data.

\subsection{Comparison Results}

\noindent\textbf{Object Classification.}
Table~\ref{tab:com_cls_SONN_MN40} presents a comparative analysis of the APPT classification performance compared to the existing methods in the ScanObjectNN and ModelNet40 datasets.
From the experimental results, the following observations can be drawn:
1) The integration of pre-trained models, irrespective of modality, consistently enhances model performance, albeit with varying degrees of improvement across methods. 
For example, incorporating the 3D pre-trained model Transformer-OcCo improves performance by 1.6\% on ScanObjectNN and 0.7\% on ModelNet40, demonstrating the effectiveness of leveraging 3D priors.
In contrast, Joint-MAE achieves more substantial improvements of 8.9\% and 2.6\% on the respective datasets. 
These results underscore the necessity for developing more effective strategies to better harness prior knowledge and maximize performance gains.
2) APPT consistently outperforms existing SOTA methods by a large margin, particularly on the challenging real-world dataset, ScanObjectNN.
For example, on the most challenging split, PB-T50-RS, the recent method Any2Point, which also employs Point-PN for point cloud tokenization, achieves accuracies of 87.7\% with the visual pre-trained model and 91.9\% with the textual pre-trained model, improving 0.7\%  and 4.8\%, respectively, over Point-PN. 
In comparison, APPT achieves accuracies of 92.6\% with the visual pre-trained model and 91.4\% with the textual pre-trained model, delivering remarkable gains of 5.5\% and 4.3\%, respectively, over Point-PN.
On ModelNet40, APPT outperforms Any2Point across all corresponding pre-trained modalities and surpasses other SOTA competitors. 
For instance, APPT with the textual pre-trained model achieves an accuracy of 95.1\%, exceeding Any2Point by 0.8\% and ReCon, which integrates 3D+2D+1D pre-trained modalities, by 1.7\%.
Overall, our method demonstrates superior performance.


\noindent\textbf{Few-shot Classification.}
To demonstrate the generalization capability of the proposed APPT, we conduct experiments under few-shot settings, following the common protocol established in~\cite{yu2022point,GuoZQLH23JointMAE}.
The `$N$-way, $K$-shot' configuration is a conventional setup, wherein $N$ classes are randomly selected, with each class containing $K$ training samples and 20 testing samples. 
Each experimental setting was repeated 10 times, and the results are reported as the mean performance accompanied by the standard deviation. 
The results are summarized in Table~\ref{tab:com_few_shot}.
Compared to both 2D and 3D pre-trained models, APPT exhibits superior generalization ability in few-shot learning.
For instance, APPT achieves notable improvements of 3.0\%, 3.2\%, 3.3\%, and 2.9\% over the 3D pre-trained model Transformer-OcCo in four distinct settings.
Furthermore, even compared to recently proposed SOTA methods such as Point-MAE, Joint-MAE, and our conference version APF, APPT consistently outperforms these approaches in terms of both accuracy and stability.
The only exception occurs in the 10-way 20-shot setting, where APPT marginally underperforms compared to APT. 
These results underscore the robustness and efficacy of the proposed APPT framework in few-shot learning tasks.

\noindent\textbf{Part Segmentation.}
In alignment with established methodologies~\cite{qi2017pointnet,pangYT22PointMAE,GuoZQLH23JointMAE}, we sample 2,048 points from each input instance and adopt the same segmentation head as utilized in Point-MAE~\cite{pangYT22PointMAE} and Joint-MAE~\cite{GuoZQLH23JointMAE}.
The corresponding results are detailed in Table~\ref{tab:com_seg}. 
Although APPT may not outperform SOTA methods across all evaluation metrics, it exhibits competitive overall performance.
Notably, APPT outperforms both P2P and our conference version APF, both of which leverage image priors, underscoring its enhanced capability in integrating multimodal information.
Furthermore, although APPT marginally lags behind Joint-MAE in terms of $\text{mIoU}_C$ and $\text{mIoU}_I$, it is crucial to emphasize that Joint-MAE necessitates training from scratch, a process that demands substantially greater computational resources and training time.
In contrast, APPT requires significantly lower computational overhead, making it a more efficient and practical alternative for segmentation tasks.

\subsection{Further Analysis}

\begin{table}[!t]
\caption{Impact of each component. 
The abbreviations are defined as follows: PE: pint embedding, 2D Mod.: 2D modality, PPT: point prompt tuning, SONN: ScanObjectNN, and MN40: ModelNet40.
}  \label{tab:abla}
 \centering  
 \setlength{\tabcolsep}{8pt}
 \renewcommand{\arraystretch}{1.1}
 \resizebox{1.\linewidth}{!}
 {
  \begin{tabular}{cccc|cc}
  \hlinew{1pt} 
  \multirow{2}{*}{PE} & \multirow{2}{*}{2D Mod.} & \multirow{2}{*}{PosIn} & \multirow{2}{*}{PPT} &SONN  & MN40 \\
  \cline{5-6}
    &  &   &     & Acc. (\%)  & Acc. (\%)  \\
  \hline
  \ding{51}  & \ding{55} &\ding{55}  & \ding{55} & 87.1 (base)     & 93.8 (base) \\  
  \ding{51}  & \ding{51} &\ding{55}  & \ding{55} & 90.1 \color{deepred}($\uparrow$ 3.0) &  93.9 \color{deepred}($\uparrow$ 0.1)\\ 
  \ding{51}  & \ding{51} & \ding{51}  & \ding{55} & 91.2 \color{deepred}($\uparrow$ 4.1)          & 94.1 \color{deepred}($\uparrow$ 0.3)\\
  \ding{51}  & \ding{51} & \ding{55}  & \ding{51} & 91.4 \color{deepred}($\uparrow$ 4.3)          & 94.1 \color{deepred}($\uparrow$ 0.3)\\  
  \ding{51}  & \ding{51} & \ding{51}  & \ding{51} & \textbf{92.6} \color{deepred}($\uparrow$ 5.5) & \textbf{94.2} \color{deepred}($\uparrow$ 0.4) \\ 
  \hlinew{1pt}
 \end{tabular}
 }
\end{table}

\begin{figure}[!t]
\subfloat[PE]{ 
    \includegraphics[width=0.5\linewidth, height=0.4\linewidth]{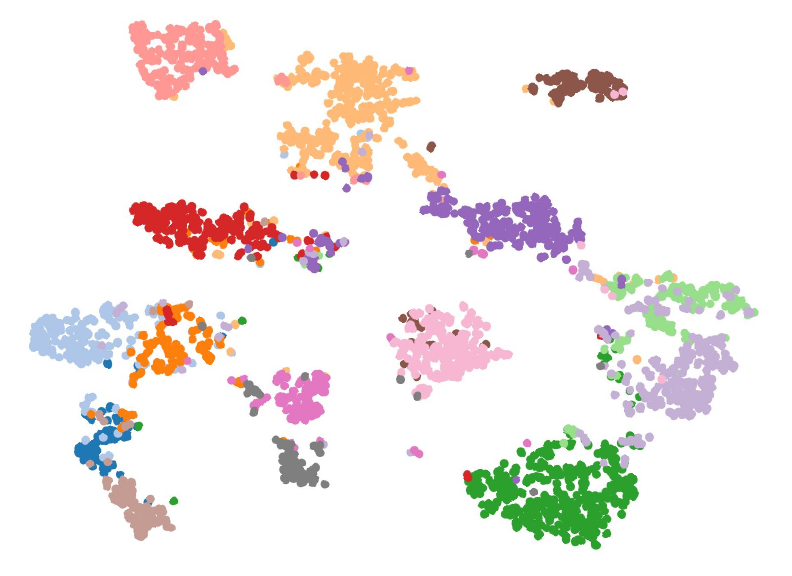}
    \label{fig:PN}
    }
\subfloat[APPT w/o PPT]{   
    \includegraphics[width=0.5\linewidth, height=0.4\linewidth]{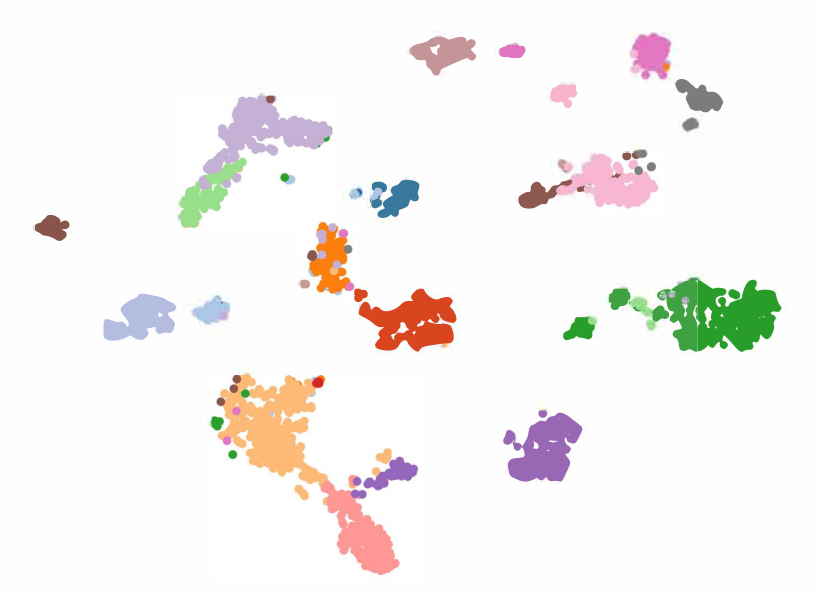}
    \label{fig:PI}
    }  \\
\subfloat[APPT w/o PosIn]{ 
    \includegraphics[width=0.5\linewidth, height=0.4\linewidth]{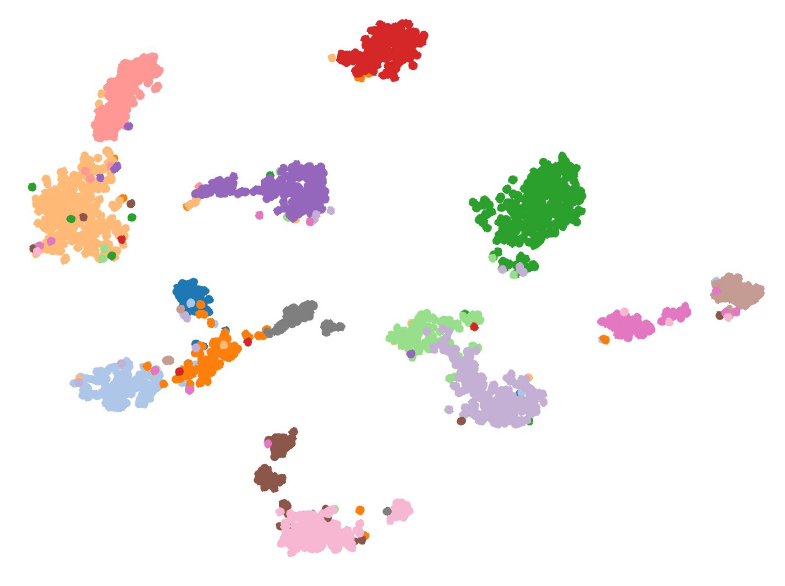}
    \label{fig:PT}
    }    
\subfloat[APPT]{
    \includegraphics[width=0.5\linewidth, height=0.4\linewidth]{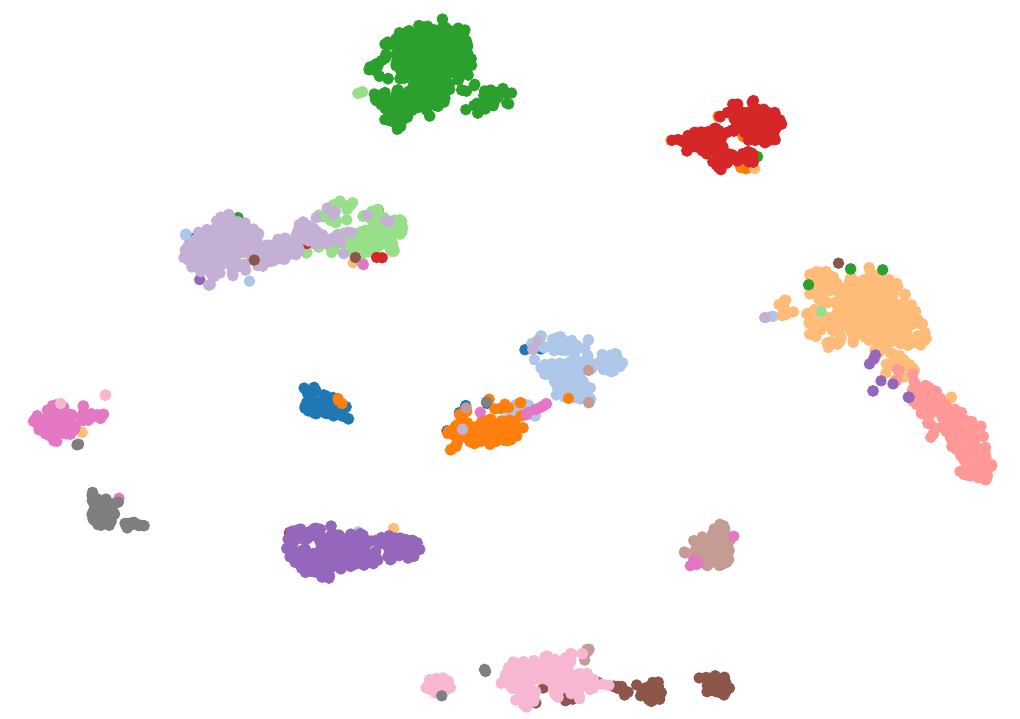}
    \label{fig:APPT}
    }   
\caption{T-SNE visualization of feature distributions. We show the results on the test set of ScanObjectNN.}
\label{fig:tsne}
\end{figure}

\begin{figure}[!t]
    \centering
    \includegraphics[width=1.\linewidth]{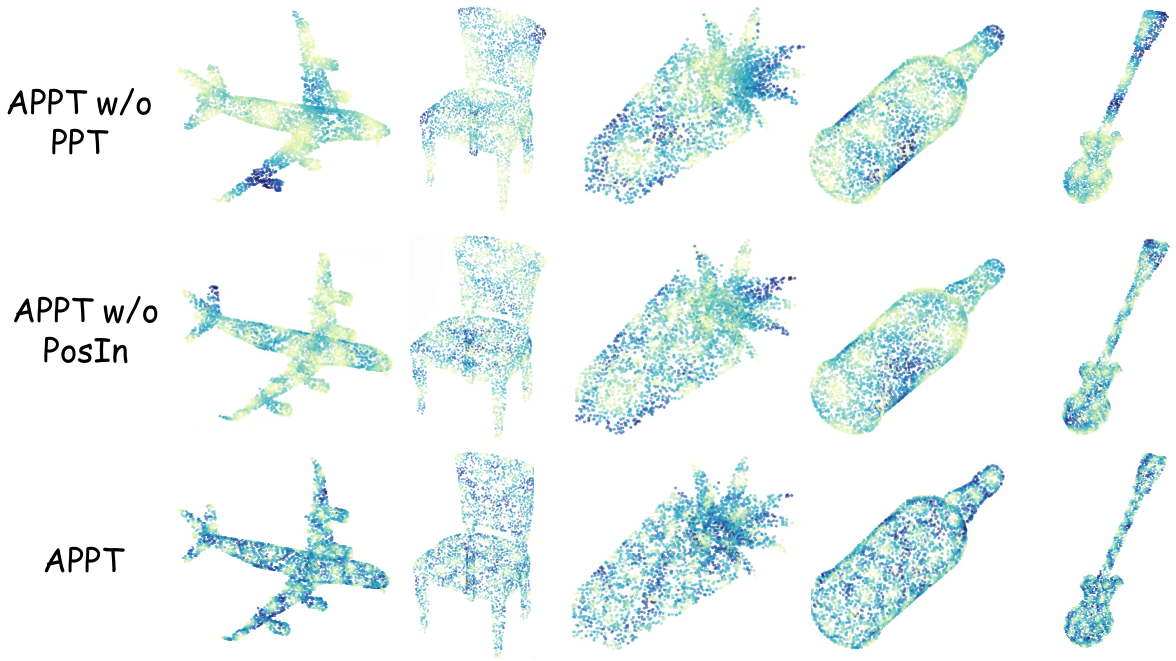}
    \caption{Visualization of the effectiveness of different modules.
    The blue color represents a higher response.}
    \label{fig:scatter}
\end{figure}

\noindent\textbf{Ablation Study of Individual Modules.}  
To systematically assess the contribution of each module within APPT, we conduct controlled experiments, with the experimental settings and results comprehensively outlined in Table~\ref{tab:abla}. 
The results demonstrate that each module plays a crucial role in enhancing the performance of the baseline method, which employs the point embedding (PE) module based on Point-PN. 
Notably, the pre-trained model on the 2D modality, along with the point-prompt tuning (PPT) and position injection (PosIn) modules yield substantial performance improvements across both datasets, highlighting their pivotal contributions to the overall effectiveness of APPT.

To further elucidate the contribution of each module, we visualize the feature distribution and the corresponding response on the original input point clouds, as shown in Figs.~\ref{fig:tsne} and \ref{fig:scatter}, respectively.
Specifically,  when the point embedding module (PE, namely Point-PN) is employed independently, the feature distribution across categories exhibits overlap, as shown in Fig.~\ref{fig:PN}. 
Fig.~\ref{fig:PI} illustrates the feature distribution after the PosIn module aligns with the ViT-B architecture, which is built upon the Point-PN framework and leverages the 2D pre-trained model. 
Meanwhile, Fig.~\ref{fig:PT} demonstrates the effectiveness of PPT module, utilizing the same PE module and 2D pre-trained model. 
When combined with the visualizations in Fig.~\ref{fig:scatter}, it becomes evident that PPT and PosIn focus on distinct regions of the object; however, both modules emphasize the object structure, thereby enhancing the separability of the learned representations.
This complementary focus underscores the synergistic contribution of PPT and PosIn to the overall performance of the framework.
Finally, Fig.~\ref{fig:APPT} demonstrates the combined effect of APPT. 
The third row of Fig.~\ref{fig:scatter} reveals that APPT captures a relatively complete and coherent overall structure of the object. 
This observation helps explain why APPT achieves significant improvement in classification but performs slightly inferior to Joint-MAE in segmentation, as APPT encoder prioritizes the global structure of the input over fine-grained local details.
Intuitively, the feature distribution boundaries obtained by APPT are more distinct, with a notable enhancement in feature separation.

\begin{table}[t]
 \caption{Performance comparison of different pre-trained modality (Pre. Mod.) on ScanObjectNN (PB-T50-RS).}  \label{tab:checkpoint}
 \centering  
 \setlength{\tabcolsep}{9pt}
 \renewcommand{\arraystretch}{1.2}
 \resizebox{1.\linewidth}{!}
 {
  \begin{tabular}{c|ccc}
  \hlinew{1pt}
  Method & Pre. Mod. & Model & Acc. (\%) \\
  \hline
   APF~\cite{Li2024APF} & 2D &  ViT-B~\cite{Dosovitskiy21vit} & 87.8 \\
    \hdashline
        & 2D & DINOv2~\cite{Oquab2023DINOv2LR} & 87.7\\
        & 2D & DeiT~\cite{Touvron2020TrainingDI} & 87.3\\
  Any2Point~\cite{tang2025any2point} &  1D (Aud.) & ImageBind~\cite{girdhar2023imagebind} & 87.0\\
     & 1D (Text) & CLIP~\cite{radford2021clip} &  91.9\\ 
     & 1D (Text) & RoBERTa~\cite{Liu2019RoBERTaAR} &  89.7\\ 
  \hdashline
   & 2D & ViT-B~\cite{Dosovitskiy21vit} & 92.6\\
   & 2D & DINOv2~\cite{Oquab2023DINOv2LR} & 92.6\\
  APPT (Ours) & 2D & DeiT~\cite{Touvron2020TrainingDI} & 88.9\\
  & 1D (Aud.) & ImageBind~\cite{girdhar2023imagebind} &  88.9\\ 
  & 1D (Text) & CLIP~\cite{radford2021clip} &  91.4\\ 
   & 1D (Text) & RoBERTa~\cite{Liu2019RoBERTaAR} &  87.3\\ 
  \hlinew{1pt}
 \end{tabular}
 }
\end{table}
\begin{table}[t]
 \caption{Comparison results of different pre-trained models on ScanObjectNN PB-T50-RS (SONN) and ModelNet~(MN40) datasets. }  \label{tab:analysis}
 \centering  
 \setlength{\tabcolsep}{10pt}
 \renewcommand{\arraystretch}{1.2}
 \resizebox{0.9\linewidth}{!}
 {
  \begin{tabular}{l|cccc}
  \hlinew{1pt}
  \multirow{2}{*}{Method}   & \multirow{2}{*}{Model} & SONN & MN40  \\
  \cline{3-4}
              &             & Acc. (\%) & Acc. (\%) \\
  \hline
  Point-PN            & N/A  & 87.1  & 93.8   \\
  Transformer          & N/A & 77.2   & 91.4   \\
  APPT \textit{w. 2D}  & ViT-B & 92.6 \color{deepred}($\uparrow$ 5.5) & 94.2 \color{deepred}($\uparrow$ 0.4)\\ 
  APPT \textit{w. 1D}  & CLIP  & 91.4 \color{deepred}($\uparrow$ 4.3) & 95.1 \color{deepred}($\uparrow$ 1.3)\\  
 \hlinew{1pt}
 \end{tabular}
 }
\end{table}

\begin{table}[t]
\caption{Performance comparison w.r.t. trainable parameters number (\# Tr. param.) on ScanObjectNN (PB-T50-RS). 
}  \label{tab:params}  
 \centering  
 \renewcommand{\arraystretch}{1.12}
 \resizebox{1.\linewidth}{!}
 {
  \begin{tabular}{l|ccc}    
  \hlinew{1pt}
   Method & Pre. Mod.  & \# Tr. Param. & Acc. (\%)\\
  \hline
  PointNet++  &  N/A & 1.4M & 77.9\\  
  PointMLP & N/A   & 12.6M & 85.2\\ 
  \hdashline
  DGCNN-OcCo & 3D & 1.8M & 83.9\\
  Point-BERT& 3D  & 21.1M & 83.1\\
  Point-MAE & 3D  & 21.1M & 85.2\\
  \hdashline
  P2P w. ViT-B & 2D  & \underline{\textbf{0.25M}} & 84.1\\
  \makecell[l]{P2P w.\\ \: HorNet-L-22k-mlp}
    & 2D  & 1.2M  & 89.3\\
  Any2Point & 2D & \textbf{0.8M} & 87.7 \\
  APF w. PointNet & 2D & 2.4M & 83.1 \\
  APF w. PointMLP & 2D  & 5.8M & \textbf{87.8}\\   
  APPT (ours) & 2D  & 3.4M & \underline{\textbf{92.6}}\\  
  \hlinew{1pt}
 \end{tabular}
 }
\end{table}

\noindent\textbf{The Impact of Different Foundation Models.}
We compare the performance of APPT across different pre-trained foundation models on ScanObjectNN PB-T50-RS dataset. 
The corresponding results are summarized in Table~\ref{tab:checkpoint}, which also includes comparisons with other methods using the same pre-trained models. 
Except when leveraging textual pre-trained knowledge, APPT consistently outperforms all other methods with the same pre-trained foundation models.
For example, with the 2D prior, APPT outperforms APF (pre-trained on ViT-B) by 4.8\% and Any2Point (pre-trained on DINOv2) by 4.9\%.
Although APPT slightly lags behind Any2Point when using textual pre-trained knowledge on ScanObjectNN (91.4\% vs. 91.7\% and 87.3\% vs. 89.7\%), it significantly outperforms Any2Point when leveraging the 1D audio prior (88.9\% vs. 87.0\%). 
On ModelNet40 (see Table~\ref{tab:com_cls_SONN_MN40}), APPT also achieves superior performance compared to Any2Point with text prior (95.1\% vs. 94.3\%).
Furthermore, experiments with other pre-trained base models, such as DeiT~\cite{Touvron2020TrainingDI} (visual prior) and ImageBind~\cite{girdhar2023imagebind} (audio prior), show that APPT consistently outperforms the baseline method (88.9\% and 88.9\% vs. 87.1\%) by a clear margin. 
Additionally, Table~\ref{tab:analysis} provides a comparison of APPT with baseline methods, demonstrating its performance improvement with the use of multiple modalities.
These results underscore the robustness and versatility of APPT across diverse pre-trained models and modalities.

\noindent\textbf{Comparison of Trainable Parameters.}
Table~\ref{tab:params} provides a comparison of APPT with SOTA methods based on pre-trained foundation models, with a focus on the number of trainable parameters.
In contrast to P2P and Any2Point, our method introduces more parameters during point token embedding, yet yields a notable performance improvement. 
On the other hand, APPT significantly reduces the number of training parameters compared to Point-MAE and Point-BERT, while simultaneously delivering notable performance gains, attributed to its efficient fine-tuning strategy. 
Additionally, compared to APF, APPT further reduces the number of trainable parameters and improves model performance through the implementation of a shared weights strategy.
Improving the efficiency of training parameters will remain a primary focus of our future research.

\section{Conclusion}
This paper has proposed an innovative PEFT architecture, APPT, designed to effectively leverage diverse pre-trained foundation models for 3D understanding tasks.
It leverages the rich semantic information embedded in large pre-trained models to efficiently enhance 3D understanding tasks, thereby addressing the challenges of data scarcity and overfitting often faced by 3D pre-trained models.
APPT departs from the existing projection-based method by adopting a point embedding module to maximize the retention of high-dimensional structural information from point clouds. 
A permutation-invariant feature is then utilized to determine the relative positions of point embeddings, enhancing the understanding of point cloud structures while effectively leveraging the priors embedded in heterogeneous pre-trained models.
The attention mechanism of the pre-trained large model is adapted through point-prompts generated by a shared weights prompt generator, ensuring efficient and scalable integration of pre-trained knowledge.
Extensive experiments have demonstrated that APPT exhibits strong generalization capabilities across various heterogeneous foundation models, achieving significant performance improvements in 3D understanding tasks.

\bibliographystyle{IEEEtran}
\bibliography{reference}

\begin{thebibliography}{10}
\providecommand{\url}[1]{#1}
\csname url@samestyle\endcsname
\providecommand{\newblock}{\relax}
\providecommand{\bibinfo}[2]{#2}
\providecommand{\BIBentrySTDinterwordspacing}{\spaceskip=0pt\relax}
\providecommand{\BIBentryALTinterwordstretchfactor}{4}
\providecommand{\BIBentryALTinterwordspacing}{\spaceskip=\fontdimen2\font plus
\BIBentryALTinterwordstretchfactor\fontdimen3\font minus \fontdimen4\font\relax}
\providecommand{\BIBforeignlanguage}[2]{{%
\expandafter\ifx\csname l@#1\endcsname\relax
\typeout{** WARNING: IEEEtran.bst: No hyphenation pattern has been}%
\typeout{** loaded for the language `#1'. Using the pattern for}%
\typeout{** the default language instead.}%
\else
\language=\csname l@#1\endcsname
\fi
#2}}
\providecommand{\BIBdecl}{\relax}
\BIBdecl

\bibitem{chen2020tuning}
T.~Chen, S.~Kornblith, M.~Norouzi, and G.~E. Hinton, ``A simple framework for contrastive learning of visual representations,'' in \emph{Int. Conf. Mach. Learn.}, 2020, pp. 1597--1607.

\bibitem{HuSWALWWC22LoRA}
E.~J. Hu, Y.~Shen, P.~Wallis, Z.~Allen-Zhu, Y.~Li, S.~Wang, L.~Wang, and W.~Chen, ``{LoRA}: Low-rank adaptation of large language models,'' in \emph{Int. Conf. Learn. Represent.}, 2022.

\bibitem{chen2022adaptformer}
S.~Chen, C.~Ge, Z.~Tong, J.~Wang, Y.~Song, J.~Wang, and P.~Luo, ``Adaptformer: Adapting vision transformers for scalable visual recognition,'' \emph{Adv. Neural Inform. Process. Syst.}, vol.~35, pp. 16\,664--16\,678, 2022.

\bibitem{yu2023visual}
B.~X. Yu, J.~Chang, H.~Wang, L.~Liu, S.~Wang, Z.~Wang, J.~Lin, L.~Xie, H.~Li, Z.~Lin \emph{et~al.}, ``Visual tuning,'' \emph{{ACM} Comput. Surv.}, vol.~56, no.~12, pp. 297:1--297:38, 2024.

\bibitem{liu2023pre}
P.~Liu, W.~Yuan, J.~Fu, Z.~Jiang, H.~Hayashi, and G.~Neubig, ``Pre-train, prompt, and predict: A systematic survey of prompting methods in natural language processing,'' \emph{{ACM} Comput. Surv.}, vol.~55, no.~9, pp. 1--35, 2023.

\bibitem{devlin2018bert}
J.~Devlin, M.-W. Chang, K.~Lee, and K.~Toutanova, ``Bert: Pre-training of deep bidirectional transformers for language understanding,'' \emph{arXiv preprint arXiv:1810.04805}, 2018.

\bibitem{Floridi2020GPT-3}
L.~Floridi and M.~Chiriatti, ``Gpt-3: Its nature, scope, limits, and consequences,'' \emph{Minds and Machines}, vol.~30, pp. 681--694, 2020.

\bibitem{Dosovitskiy21vit}
D.~Alexey, B.~Lucas, K.~Alexander, W.~Dirk, Z.~Xiaohua, U.~Thomas, D.~Mostafa, M.~Matthias, H.~Georg, G.~Sylvain \emph{et~al.}, ``An image is worth 16x16 words: Transformers for image recognition at scale,'' in \emph{Int. Conf. Learn. Represent.}, 2021.

\bibitem{radford2021clip}
A.~Radford, J.~W. Kim, C.~Hallacy, A.~Ramesh, G.~Goh, S.~Agarwal, G.~Sastry, A.~Askell, P.~Mishkin, J.~Clark \emph{et~al.}, ``Learning transferable visual models from natural language supervision,'' in \emph{Int. Conf. Learn. Represent.}, 2021, pp. 8748--8763.

\bibitem{Oquab2023DINOv2LR}
M.~Oquab, T.~Darcet, T.~Moutakanni, H.~Vo, M.~Szafraniec, V.~Khalidov, P.~Fernandez, D.~Haziza, F.~Massa, A.~El-Nouby \emph{et~al.}, ``{DINOv2}: Learning robust visual features without supervision,'' \emph{Trans. Mach. Learn. Res.}, vol. 2024, 2024.

\bibitem{guo2020deep}
Y.~Guo, H.~Wang, Q.~Hu, H.~Liu, L.~Liu, and M.~Bennamoun, ``Deep learning for 3d point clouds: A survey,'' \emph{IEEE Trans. Pattern Anal. Mach. Intell.}, vol.~43, no.~12, pp. 4338--4364, 2020.

\bibitem{yu2022point}
X.~Yu, L.~Tang, Y.~Rao, T.~Huang, J.~Zhou, and J.~Lu, ``Point-bert: Pre-training 3d point cloud transformers with masked point modeling,'' in \emph{IEEE/CVF Conf. Comput. Vis. Pattern Recog.}, 2022, pp. 19\,313--19\,322.

\bibitem{wangHC2021occo}
H.~Wang, Q.~Liu, X.~Yue, J.~Lasenby, and M.~J. Kusner, ``Unsupervised point cloud pre-training via occlusion completion,'' in \emph{Int. Conf. Comput. Vis.}, 2021, pp. 9782--9792.

\bibitem{NEURIPS2023pointgpt}
G.~Chen, M.~Wang, Y.~Yang, K.~Yu, L.~Yuan, and Y.~Yue, ``Pointgpt: Auto-regressively generative pre-training from point clouds,'' in \emph{Adv. Neural Inform. Process. Syst.}, vol.~36, 2023, pp. 29\,667--29\,679.

\bibitem{tang2025any2point}
Y.~Tang, R.~Zhang, J.~Liu, Z.~Guo, B.~Zhao, Z.~Wang, P.~Gao, H.~Li, D.~Wang, and X.~Li, ``Any2point: Empowering any-modality large models for efficient 3d understanding,'' in \emph{Eur. Conf. Comput. Vis.}, 2024, pp. 456--473.

\bibitem{Ziyi21P2P}
Z.~Wang, X.~Yu, Y.~Rao, J.~Zhou, and J.~Lu, ``{P2P:} tuning pre-trained image models for point cloud analysis with point-to-pixel prompting,'' \emph{Adv. Neural Inform. Process. Syst.}, 2022.

\bibitem{Zhang2023Flattening-Net}
Q.~Zhang, J.~Hou, Y.~Qian, Y.~Zeng, J.~Zhang, and Y.~He, ``Flattening-net: Deep regular 2d representation for 3d point cloud analysis,'' \emph{IEEE Trans. Pattern Anal. Mach. Intell.}, vol.~45, no.~8, pp. 9726--9742, 2023.

\bibitem{Wang2024point-to-pixel}
Z.~Wang, Y.~Rao, X.~Yu, J.~Zhou, and J.~Lu, ``Point-to-pixel prompting for point cloud analysis with pre-trained image models,'' \emph{IEEE Trans. Pattern Anal. Mach. Intell.}, vol.~46, no.~6, pp. 4381--4397, 2024.

\bibitem{XuRS2024PointLLM}
R.~Xu, X.~Wang, T.~Wang, Y.~Chen, J.~Pang, and D.~Lin, ``Pointllm: Empowering large language models to understand point clouds,'' in \emph{Eur. Conf. Comput. Vis.}, 2024, pp. 131--147.

\bibitem{ZhangI2PMAE23}
R.~Zhang, L.~Wang, Y.~Qiao, P.~Gao, and H.~Li, ``Learning {3D} representations from {2D} pre-trained models via image-to-point masked autoencoders,'' in \emph{IEEE/CVF Conf. Comput. Vis. Pattern Recog.}, 2023, pp. 21\,769--21\,780.

\bibitem{Liu2023OpenShape}
M.~Liu, R.~Shi, K.~Kuang, Y.~Zhu, X.~Li, S.~Han, H.~Cai, F.~Porikli, and H.~Su, ``Openshape: Scaling up 3d shape representation towards open-world understanding,'' in \emph{Adv. Neural Inform. Process. Syst.}, vol.~36, 2023, pp. 44\,860--44\,879.

\bibitem{Umam2024PartDistill}
A.~Umam, C.-K. Yang, M.-H. Chen, J.-H. Chuang, and Y.-Y. Lin, ``Partdistill: 3d shape part segmentation by vision-language model distillation,'' in \emph{IEEE/CVF Conf. Comput. Vis. Pattern Recog.}, 2024, pp. 3470--3479.

\bibitem{xue2023ulip}
L.~Xue, M.~Gao, C.~Xing, R.~Mart{\'\i}n-Mart{\'\i}n, J.~Wu, C.~Xiong, R.~Xu, J.~C. Niebles, and S.~Savarese, ``{ULIP}: Learning a unified representation of language, images, and point clouds for 3d understanding,'' in \emph{IEEE/CVF Conf. Comput. Vis. Pattern Recog.}, 2023, pp. 1179--1189.

\bibitem{Xue2024ulip2}
L.~Xue, N.~Yu, S.~Zhang, A.~Panagopoulou, J.~Li, R.~Mart{\'\i}n-Mart{\'\i}n, J.~Wu, C.~Xiong, R.~Xu, J.~C. Niebles, and S.~Savarese, ``{ULIP}-2: Towards scalable multimodal pre-training for 3d understanding,'' in \emph{IEEE/CVF Conf. Comput. Vis. Pattern Recog.}, June 2024, pp. 27\,091--27\,101.

\bibitem{QiNIPS2017pointnet2}
C.~R. Qi, L.~Yi, H.~Su, and L.~J. Guibas, ``Pointnet++: Deep hierarchical feature learning on point sets in a metric space,'' \emph{Adv. Neural Inform. Process. Syst.}, vol.~30, 2017.

\bibitem{zaheer2017deep}
M.~Zaheer, S.~Kottur, S.~Ravanbakhsh, B.~Poczos, R.~R. Salakhutdinov, and A.~J. Smola, ``Deep sets,'' \emph{Adv. Neural Inform. Process. Syst.}, vol.~30, 2017.

\bibitem{Li2024APF}
M.~Li, D.~Li, G.~Yang, Y.~Cheung, and H.~Huang, ``Adapt pointformer: 3d point cloud analysis via adapting 2d visual transformers,'' in \emph{Eur. Conf. Artif. Intell.}, vol. 392, 2024, pp. 89--96.

\bibitem{qi2017pointnet}
C.~R. Qi, H.~Su, K.~Mo, and L.~J. Guibas, ``Pointnet: Deep learning on point sets for 3d classification and segmentation,'' in \emph{IEEE/CVF Conf. Comput. Vis. Pattern Recog.}, 2017, pp. 652--660.

\bibitem{Liu2019pointvoxel}
Z.~Liu, H.~Tang, Y.~Lin, and S.~Han, ``Point-voxel cnn for efficient 3d deep learning,'' \emph{Adv. Neural Inform. Process. Syst.}, vol.~32, 2019.

\bibitem{Shi2020CVPR}
S.~Shi, C.~Guo, L.~Jiang, Z.~Wang, J.~Shi, X.~Wang, and H.~Li, ``Pv-rcnn: Point-voxel feature set abstraction for 3d object detection,'' in \emph{IEEE/CVF Conf. Comput. Vis. Pattern Recog.}, 2020, pp. 10\,529--10\,538.

\bibitem{ranH2022surface}
H.~Ran, J.~Liu, and C.~Wang, ``Surface representation for point clouds,'' in \emph{IEEE/CVF Conf. Comput. Vis. Pattern Recog.}, 2022, pp. 18\,942--18\,952.

\bibitem{li2023bevdepth}
Y.~Li, Z.~Ge, G.~Yu, J.~Yang, Z.~Wang, Y.~Shi, J.~Sun, and Z.~Li, ``Bevdepth: Acquisition of reliable depth for multi-view 3d object detection,'' in \emph{AAAI Conf. Artif. Intell.}, vol.~37, no.~2, 2023, pp. 1477--1485.

\bibitem{qianG2022pointnext}
G.~Qian, Y.~Li, H.~Peng, J.~Mai, H.~Hammoud, M.~Elhoseiny, and B.~Ghanem, ``Pointnext: Revisiting pointnet++ with improved training and scaling strategies,'' \emph{Adv. Neural Inform. Process. Syst.}, vol.~35, pp. 23\,192--23\,204, 2022.

\bibitem{Wang2019Dynamic}
Y.~Wang, Y.~Sun, Z.~Liu, S.~E. Sarma, M.~M. Bronstein, and J.~M. Solomon, ``Dynamic graph cnn for learning on point clouds,'' \emph{ACM Trans. Graph.}, vol.~38, no.~5, pp. 1--12, 2019.

\bibitem{ThomasH2019KPConv}
H.~Thomas, C.~R. Qi, J.-E. Deschaud, B.~Marcotegui, F.~Goulette, and L.~J. Guibas, ``Kpconv: Flexible and deformable convolution for point clouds,'' in \emph{Int. Conf. Comput. Vis.}, 2019, pp. 6411--6420.

\bibitem{vaswani2017attention}
A.~Vaswani, N.~Shazeer, N.~Parmar, J.~Uszkoreit, L.~Jones, A.~N. Gomez, {\L}.~Kaiser, and I.~Polosukhin, ``Attention is all you need,'' \emph{Adv. Neural Inform. Process. Syst.}, vol.~30, 2017.

\bibitem{zhao2021point}
H.~Zhao, L.~Jiang, J.~Jia, P.~H. Torr, and V.~Koltun, ``Point transformer,'' in \emph{Int. Conf. Comput. Vis.}, 2021, pp. 16\,259--16\,268.

\bibitem{guoMH2021pct}
M.-H. Guo, J.-X. Cai, Z.-N. Liu, T.-J. Mu, R.~R. Martin, and S.-M. Hu, ``Pct: Point cloud transformer,'' \emph{Comput. Vis. Media}, vol.~7, pp. 187--199, 2021.

\bibitem{choe2022pointmixer}
J.~Choe, C.~Park, F.~Rameau, J.~Park, and I.~S. Kweon, ``Pointmixer: Mlp-mixer for point cloud understanding,'' in \emph{Eur. Conf. Comput. Vis.}, 2022, pp. 620--640.

\bibitem{wu2022ptv2}
X.~Wu, Y.~Lao, L.~Jiang, X.~Liu, and H.~Zhao, ``Point transformer v2: Grouped vector attention and partition-based pooling,'' \emph{Adv. Neural Inform. Process. Syst.}, vol.~35, pp. 33\,330--33\,342, 2022.

\bibitem{duan2023condaformer}
L.~Duan, S.~Zhao, N.~Xue, M.~Gong, G.-S. Xia, and D.~Tao, ``Condaformer: Disassembled transformer with local structure enhancement for 3d point cloud understanding,'' \emph{Adv. Neural Inform. Process. Syst.}, vol.~36, 2023.

\bibitem{han2024mamba3d}
X.~Han, Y.~Tang, Z.~Wang, and X.~Li, ``Mamba3d: Enhancing local features for 3d point cloud analysis via state space model,'' in \emph{ACM Int. Conf. Multimedia}, 2024, pp. 4995--5004.

\bibitem{wu2024ptv3}
X.~Wu, L.~Jiang, P.-S. Wang, Z.~Liu, X.~Liu, Y.~Qiao, W.~Ouyang, T.~He, and H.~Zhao, ``Point transformer v3: Simpler faster stronger,'' in \emph{Proceedings of the IEEE/CVF Conference on Computer Vision and Pattern Recognition}, 2024, pp. 4840--4851.

\bibitem{wang2022multimodal}
Y.~Wang, X.~Chen, L.~Cao, W.~Huang, F.~Sun, and Y.~Wang, ``Multimodal token fusion for vision transformers,'' in \emph{IEEE/CVF Conf. Comput. Vis. Pattern Recog.}, 2022, pp. 12\,186--12\,195.

\bibitem{ren2024pointofview}
H.~Ren, J.~Wang, M.~Yang, and S.~Velipasalar, ``Pointofview: A multi-modal network for few-shot 3d point cloud classification fusing point and multi-view image features,'' in \emph{IEEE/CVF Conf. Comput. Vis. Pattern Recog.}, 2024, pp. 784--793.

\bibitem{li2023ashapeformer}
Z.~Li, H.~Yu, Z.~Yang, T.~Chen, and N.~Akhtar, ``Ashapeformer: Semantics-guided object-level active shape encoding for 3d object detection via transformers,'' in \emph{IEEE/CVF Conf. Comput. Vis. Pattern Recog.}, 2023, pp. 1012--1021.

\bibitem{Zheng2024Diffusion}
X.~Zheng, X.~Huang, G.~Mei, Y.~Hou, Z.~Lyu, B.~Dai, W.~Ouyang, and Y.~Gong, ``Point cloud pre-training with diffusion models,'' in \emph{IEEE/CVF Conf. Comput. Vis. Pattern Recog.}, 2024, pp. 22\,935--22\,945.

\bibitem{tang2024PointPEFT}
Y.~Tang, R.~Zhang, Z.~Guo, X.~Ma, B.~Zhao, Z.~Wang, D.~Wang, and X.~Li, ``Point-peft: Parameter-efficient fine-tuning for 3d pre-trained models,'' in \emph{AAAI Conf. Artif. Intell.}, vol.~38, no.~6, 2024, pp. 5171--5179.

\bibitem{pangYT22PointMAE}
Y.~Pang, W.~Wang, F.~E. Tay, W.~Liu, Y.~Tian, and L.~Yuan, ``Masked autoencoders for point cloud self-supervised learning,'' in \emph{Eur. Conf. Comput. Vis.}, 2022, pp. 604--621.

\bibitem{NEURIPS2022PointM2AE}
R.~Zhang, Z.~Guo, P.~Gao, R.~Fang, B.~Zhao, D.~Wang, Y.~Qiao, and H.~Li, ``Point-m2ae: Multi-scale masked autoencoders for hierarchical point cloud pre-training,'' in \emph{Adv. Neural Inform. Process. Syst.}, vol.~35, 2022, pp. 27\,061--27\,074.

\bibitem{zhang2022pointclip}
R.~Zhang, Z.~Guo, W.~Zhang, K.~Li, X.~Miao, B.~Cui, Y.~Qiao, P.~Gao, and H.~Li, ``{PointCLIP}: Point cloud understanding by clip,'' in \emph{IEEE/CVF Conf. Comput. Vis. Pattern Recog.}, 2022, pp. 8542--8552.

\bibitem{zhu2023pointclipv2}
X.~Zhu, R.~Zhang, B.~He, Z.~Guo, Z.~Zeng, Z.~Qin, S.~Zhang, and P.~Gao, ``Pointclip v2: Prompting clip and gpt for powerful 3d open-world learning,'' in \emph{Int. Conf. Comput. Vis.}, 2023, pp. 2639--2650.

\bibitem{wei2020view-Gcn}
X.~Wei, R.~Yu, and J.~Sun, ``View-{GCN}: View-based graph convolutional network for {3D} shape analysis,'' in \emph{IEEE/CVF Conf. Comput. Vis. Pattern Recog.}, 2020, pp. 1847--1856.

\bibitem{yang2020predicting}
Q.~Yang, H.~Chen, Z.~Ma, Y.~Xu, R.~Tang, and J.~Sun, ``Predicting the perceptual quality of point cloud: A 3d-to-2d projection-based exploration,'' \emph{IEEE Trans. Multimedia}, vol.~23, pp. 3877--3891, 2020.

\bibitem{yu2022data}
P.-C. Yu, C.~Sun, and M.~Sun, ``Data efficient 3d learner via knowledge transferred from 2d model,'' in \emph{Eur. Conf. Comput. Vis.}, 2022, pp. 182--198.

\bibitem{dong2023act}
R.~Dong, Z.~Qi, L.~Zhang, J.~Zhang, J.~Sun, Z.~Ge, L.~Yi, and K.~Ma, ``Autoencoders as cross-modal teachers: Can pretrained 2d image transformers help 3d representation learning?'' in \emph{Int. Conf. Learn. Represent.}, 2023.

\bibitem{jia2022visual}
M.~Jia, L.~Tang, B.-C. Chen, C.~Cardie, S.~Belongie, B.~Hariharan, and S.-N. Lim, ``Visual prompt tuning,'' in \emph{Eur. Conf. Comput. Vis.}, 2022, pp. 709--727.

\bibitem{Alexey2021vit}
A.~Dosovitskiy, L.~Beyer, A.~Kolesnikov, D.~Weissenborn, X.~Zhai, T.~Unterthiner, M.~Dehghani, M.~Minderer, G.~Heigold, S.~Gelly, and N.~H. Jakob~Uszkoreit, ``An image is worth 16x16 words: Transformers for image recognition at scale,'' in \emph{Int. Conf. Learn. Represent.}, 2021.

\bibitem{chen2021developing}
X.~Chen, Y.~Wu, Z.~Wang, S.~Liu, and J.~Li, ``Developing real-time streaming transformer transducer for speech recognition on large-scale dataset,'' in \emph{IEEE Int. Conf. Acoust. Speech Signal Process.}, 2021, pp. 5904--5908.

\bibitem{ba2016layer}
J.~L. Ba, J.~R. Kiros, and G.~E. Hinton, ``Layer normalization,'' \emph{arXiv preprint arXiv:1607.06450}, 2016.

\bibitem{he2016deep}
K.~He, X.~Zhang, S.~Ren, and J.~Sun, ``Deep residual learning for image recognition,'' in \emph{IEEE/CVF Conf. Comput. Vis. Pattern Recog.}, 2016, pp. 770--778.

\bibitem{maXQ22PointMLP}
X.~Ma, C.~Qin, H.~You, H.~Ran, and Y.~Fu, ``Rethinking network design and local geometry in point cloud: A simple residual mlp framework,'' in \emph{Int. Conf. Learn. Represent.}, 2022.

\bibitem{Zhang2023ParameterIN}
R.~Zhang, L.~Wang, Y.~Wang, P.~Gao, H.~Li, and J.~Shi, ``Starting from non-parametric networks for 3d point cloud analysis,'' in \emph{IEEE/CVF Conf. Comput. Vis. Pattern Recog.}, 2023, pp. 5344--5353.

\bibitem{Yuan21Tokens2Token}
L.~Yuan, Y.~Chen, T.~Wang, W.~Yu, Y.~Shi, Z.-H. Jiang, F.~E. Tay, J.~Feng, and S.~Yan, ``Tokens-to-token vit: Training vision transformers from scratch on imagenet,'' in \emph{Int. Conf. Comput. Vis.}, October 2021, pp. 558--567.

\bibitem{bahng2022exploring}
H.~Bahng, A.~Jahanian, S.~Sankaranarayanan, and P.~Isola, ``Exploring visual prompts for adapting large-scale models,'' \emph{arXiv preprint arXiv:2203.17274}, 2022.

\bibitem{LeeDJ2023read}
D.~Lee, S.~Song, J.~Suh, J.~Choi, S.~Lee, and H.~J. Kim, ``Read-only prompt optimization for vision-language few-shot learning,'' in \emph{IEEE/CVF Conf. Comput. Vis. Pattern Recog.}, 2023, pp. 1401--1411.

\bibitem{DongB2022lpt}
B.~Dong, P.~Zhou, S.~Yan, and W.~Zuo, ``{LPT}: Long-tailed prompt tuning for image classification,'' in \emph{Int. Conf. Learn. Represent.}, 2022.

\bibitem{shi2024LIFT}
J.-X. Shi, T.~Wei, Z.~Zhou, J.-J. Shao, X.-Y. Han, and Y.-F. Li, ``Long-tail learning with foundation model: Heavy fine-tuning hurts,'' in \emph{Forty-first International Conference on Machine Learning}, 2024.

\bibitem{li2024GNMPT}
M.~Li, Y.~Liu, Y.~Lu, Y.~Zhang, Y.-m. Cheung, and H.~Huang, ``Improving visual prompt tuning by gaussian neighborhood minimization for long-tailed visual recognition,'' \emph{arXiv preprint arXiv:2410.21042}, 2024.

\bibitem{Xiang2021PrefixTuning}
X.~L. Li and P.~Liang, ``Prefix-tuning: Optimizing continuous prompts for generation,'' in \emph{ACL/IJCNLP}, 2021, pp. 4582--4597.

\bibitem{liu2021swin}
Z.~Liu, Y.~Lin, Y.~Cao, H.~Hu, Y.~Wei, Z.~Zhang, S.~Lin, and B.~Guo, ``Swin transformer: Hierarchical vision transformer using shifted windows,'' in \emph{Int. Conf. Comput. Vis.}, 2021, pp. 10\,012--10\,022.

\bibitem{NIPS2017DeepSets}
M.~Zaheer, S.~Kottur, S.~Ravanbakhsh, B.~Poczos, R.~R. Salakhutdinov, and A.~J. Smola, ``Deep sets,'' in \emph{Adv. Neural Inform. Process. Syst.}, vol.~30, 2017.

\bibitem{NIPS2019DeepSetNet}
Y.~Zhang, J.~Hare, and A.~Prugel-Bennett, ``Deep set prediction networks,'' in \emph{Adv. Neural Inform. Process. Syst.}, vol.~32, 2019.

\bibitem{GuoZQLH23JointMAE}
G.~Ziyu, Z.~Renrui, Q.~Longtian, L.~Xianzhi, and H.~Pheng{-}Ann, ``Joint-mae: 2d-3d joint masked autoencoders for 3d point cloud pre-training,'' in \emph{Int. Joint Conf. Artif. Intell.}, 2023, pp. 791--799.

\bibitem{Xin2024DAPT}
X.~Zhou, D.~Liang, W.~Xu, X.~Zhu, Y.~Xu, Z.~Zou, and X.~Bai, ``Dynamic adapter meets prompt tuning: Parameter-efficient transfer learning for point cloud analysis,'' in \emph{IEEE/CVF Conf. Comput. Vis. Pattern Recog.}, 2024, pp. 14\,707--14\,717.

\bibitem{qi2023contrast}
Z.~Qi, R.~Dong, G.~Fan, Z.~Ge, X.~Zhang, K.~Ma, and L.~Yi, ``Contrast with reconstruct: Contrastive 3d representation learning guided by generative pretraining,'' in \emph{Int. Conf. Mach. Learn.}, 2023, pp. 28\,223--28\,243.

\bibitem{AfhamM22CrossPoint}
M.~Afham, I.~Dissanayake, D.~Dissanayake, A.~Dharmasiri, K.~Thilakarathna, and R.~Rodrigo, ``Crosspoint: Self-supervised cross-modal contrastive learning for 3d point cloud understanding,'' in \emph{IEEE/CVF Conf. Comput. Vis. Pattern Recog.}, June 2022, pp. 9902--9912.

\bibitem{uy2019revisiting}
M.~A. Uy, Q.-H. Pham, B.-S. Hua, T.~Nguyen, and S.-K. Yeung, ``Revisiting point cloud classification: A new benchmark dataset and classification model on real-world data,'' in \emph{Int. Conf. Comput. Vis.}, 2019, pp. 1588--1597.

\bibitem{wuZR20153d}
Z.~Wu, S.~Song, A.~Khosla, F.~Yu, L.~Zhang, X.~Tang, and J.~Xiao, ``{3D} shapenets: A deep representation for volumetric shapes,'' in \emph{IEEE/CVF Conf. Comput. Vis. Pattern Recog.}, 2015, pp. 1912--1920.

\bibitem{yi2016scalable}
L.~Yi, V.~G. Kim, D.~Ceylan, I.-C. Shen, M.~Yan, H.~Su, C.~Lu, Q.~Huang, A.~Sheffer, and L.~Guibas, ``A scalable active framework for region annotation in 3d shape collections,'' \emph{ACM Trans. Graph.}, vol.~35, no.~6, pp. 1--12, 2016.

\bibitem{xuM2021paconv}
M.~Xu, R.~Ding, H.~Zhao, and X.~Qi, ``Paconv: Position adaptive convolution with dynamic kernel assembling on point clouds,'' in \emph{IEEE/CVF Conf. Comput. Vis. Pattern Recog.}, 2021, pp. 3173--3182.

\bibitem{russakovsky2015imagenet}
O.~Russakovsky, J.~Deng, H.~Su, J.~Krause, S.~Satheesh, S.~Ma, Z.~Huang, A.~Karpathy, A.~Khosla, M.~Bernstein \emph{et~al.}, ``Imagenet large scale visual recognition challenge,'' \emph{Int. J. Comput. Vis.}, vol. 115, pp. 211--252, 2015.

\bibitem{girdhar2023imagebind}
R.~Girdhar, A.~El-Nouby, Z.~Liu, M.~Singh, K.~V. Alwala, A.~Joulin, and I.~Misra, ``Imagebind: One embedding space to bind them all,'' in \emph{IEEE/CVF Conf. Comput. Vis. Pattern Recog.}, 2023, pp. 15\,180--15\,190.

\bibitem{Touvron2020TrainingDI}
H.~Touvron, M.~Cord, M.~Douze, F.~Massa, A.~Sablayrolles, and H.~J{\'e}gou, ``Training data-efficient image transformers \& distillation through attention,'' in \emph{Int. Conf. Mach. Learn.}, 2021, pp. 10\,347--10\,357.

\bibitem{Liu2019RoBERTaAR}
\BIBentryALTinterwordspacing
Y.~Liu, M.~Ott, N.~Goyal, J.~Du, M.~Joshi, D.~Chen, O.~Levy, M.~Lewis, L.~Zettlemoyer, and V.~Stoyanov, ``Roberta: A robustly optimized bert pretraining approach,'' \emph{ArXiv}, vol. abs/1907.11692, 2019. [Online]. Available: \url{https://api.semanticscholar.org/CorpusID:198953378}
\BIBentrySTDinterwordspacing

\end{thebibliography}


\begin{IEEEbiography}[{\includegraphics[width=1in,height=1.25in,clip,keepaspectratio]{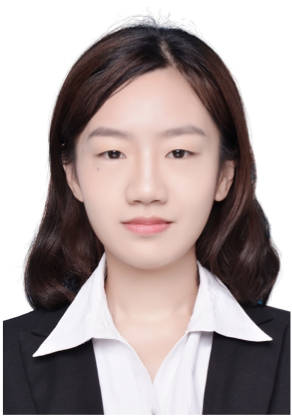}}]{Mengke Li}
received the B.S. degree in communication engineering from Southwest University, Chongqing, China, in 2015, the M.S. degree in signal and information processing from Xidian University, Xi’an, China, in 2018, and the Ph.D. degree from Hong Kong Baptist University, Hong Kong SAR, China, in 2022. She is currently an Assistant Professor with College of Computer Science and Software Engineering, Shenzhen University, Shenzhen, China. Her current research interests include imbalanced data learning, long-tail learning and computer vision.
\end{IEEEbiography}

\begin{IEEEbiography}[{\includegraphics[width=1in,height=1.25in,clip,keepaspectratio]{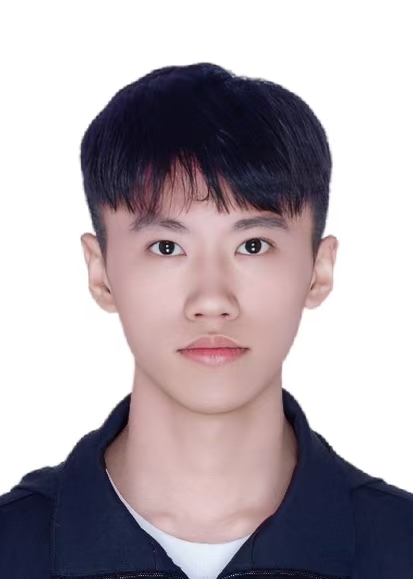}}]{Lihao Chen}
received the B.S. degree in Computer Science and Technology from Central China Normal University, Wuhan, China, in 2024. He is currently working toward the Mphil degree with Guangdong Laboratory of Artificial Intelligence and Digital Economy (Shenzhen), Shenzhen University, Guangdong, China, under the supervision of Mengke Li. His current research directions are computer vision and 3D point cloud analysis.
\end{IEEEbiography}

\begin{IEEEbiography}[{\includegraphics[width=1in,height=1.25in,clip,keepaspectratio]{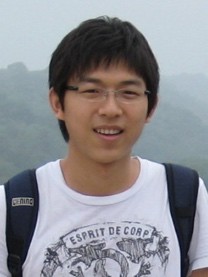}}]{Peng Zhang}
received the B.S. degree in electronic and information engineering, M.S. and Ph.D. degrees in signal and information processing from Xidian University, Xi’an, China, in 2006, 2009 and 2012 respectively. He is currently a Professor at National Key Lab. of Radar Signal Processing, Xidian University. His main research interests are SAR image interpretation and statistical learning theory.
\end{IEEEbiography}

\begin{IEEEbiography}[{\includegraphics[width=1in,height=1.25in,clip,keepaspectratio]{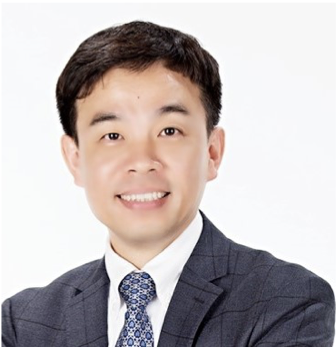}}]{Yiu-ming Cheung}
(SM'06-F'18) received the Ph.D. degree from the Department of Computer Science and Engineering at The Chinese University of Hong Kong in Hong Kong. He is a Fellow of IEEE, AAAS, IAPR, IET and BCS. He is a Chair Professor (Artificial Intelligence) of the Department of Computer Science, Hong Kong Baptist University, Hong Kong SAR, China. His research interests include machine learning and visual computing, data science, pattern recognition, multi-objective optimization, and information security. He is currently the Editor-in-Chief of IEEE Transactions on Emerging Topics in Computational Intelligence. Also, he serves as an Associate Editor for IEEE Transactions on Cybernetics, IEEE Transactions on Cognitive and Developmental Systems, IEEE Transactions on Neural Networks and Learning Systems (2014-2020), Pattern Recognition and Neurocomputing, to name a few. For details, please refer to: \href{https://www.comp.hkbu.edu.hk/~ymc}{https://www.comp.hkbu.edu.hk/\textasciitilde ymc}.
\end{IEEEbiography}

\begin{IEEEbiography}[{\includegraphics[width=1in,height=1.25in,clip,keepaspectratio]{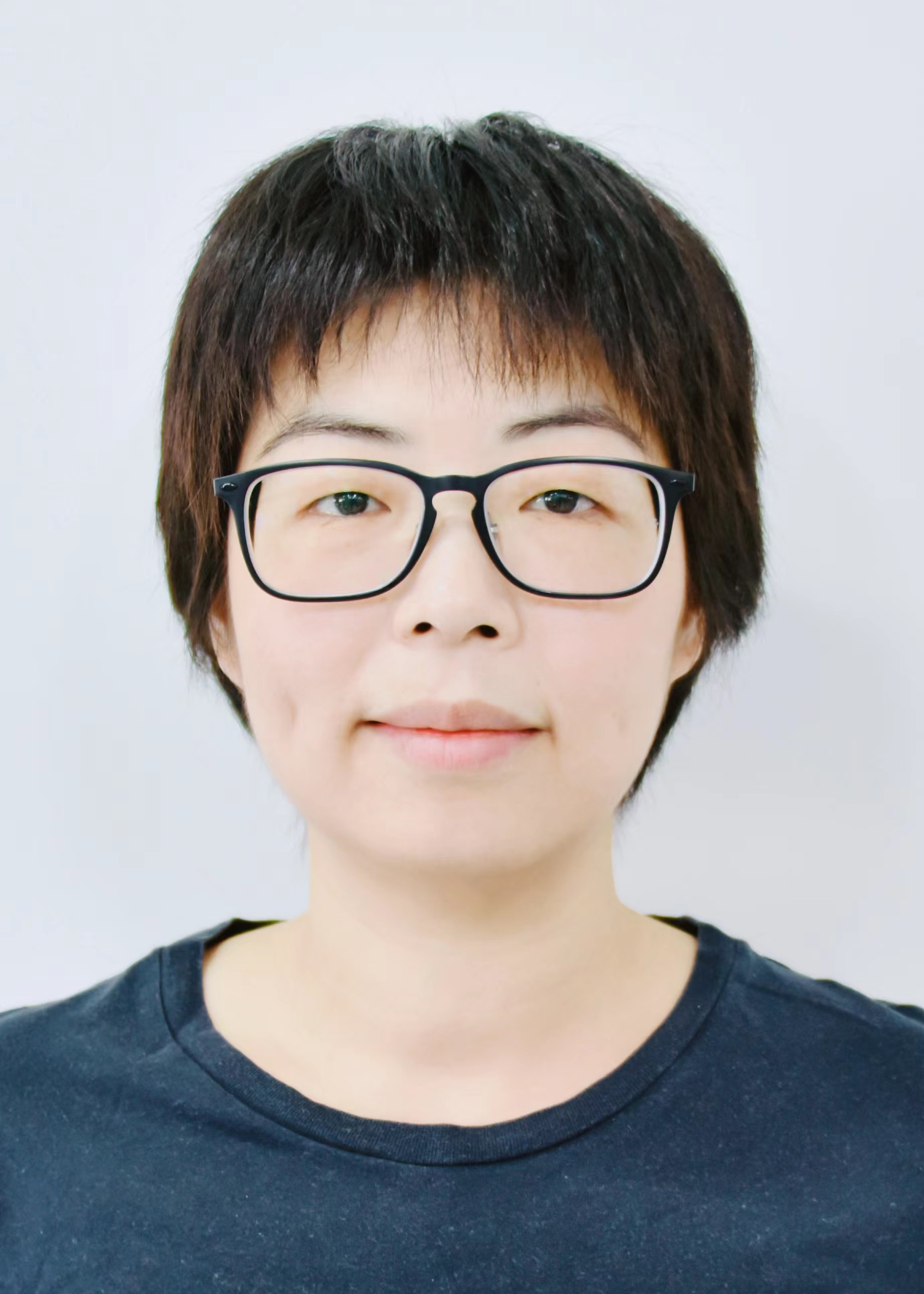}}]{Hui Huang}
received the Ph.D. degree in math from The University of British Columbia in 2008. She is Chair Professor of Computer Science at Shenzhen University, serving as the Dean of College of Computer Science and Software Engineering while also directing the Visual Computing Research Center. Her research encompasses computer graphics, computer vision and visual analytics, focusing on geometry, points, shapes and images. She is currently on the editorial board of ACM TOG and IEEE TVCG.
\end{IEEEbiography}


\end{document}